\documentclass[10pt]{article} 
\usepackage[preprint]{tmlr}

\usepackage{amsmath,amsfonts,bm}

\def\eqref#1{equation~\ref{#1}}

\def\1{\bm{1}}

\def\vx{{\bm{x}}}

\def\mA{{\bm{A}}}
\def\mB{{\bm{B}}}
\def\mC{{\bm{C}}}

\def\mI{{\bm{I}}}

\def\mL{{\bm{L}}}
\def\mM{{\bm{M}}}

\def\mQ{{\bm{Q}}}
\def\mR{{\bm{R}}}
\def\mS{{\bm{S}}}

\def\mU{{\bm{U}}}
\def\mV{{\bm{V}}}
\def\mW{{\bm{W}}}
\def\mX{{\bm{X}}}
\def\mY{{\bm{Y}}}
\def\mZ{{\bm{Z}}}

\def\mLambda{{\bm{\Lambda}}}
\def\mSigma{{\bm{\Sigma}}}

\DeclareMathAlphabet{\mathsfit}{\encodingdefault}{\sfdefault}{m}{sl}
\SetMathAlphabet{\mathsfit}{bold}{\encodingdefault}{\sfdefault}{bx}{n}

\def\gF{{\mathcal{F}}}

\usepackage{amsmath,amsthm}
\usepackage{algorithm}
\let\AND\relax
\usepackage{algorithmic}
\newtheorem{theorem}{Theorem}[section]    
\newtheorem{lemma}[theorem]{Lemma}        

\newtheorem{corollary}[theorem]{Corollary}

\theoremstyle{definition}

\theoremstyle{remark}

\newtheorem*{remark*}{Remark}

\usepackage{booktabs}
\usepackage{wrapfig}
\usepackage{multirow}
\usepackage{graphicx}
\usepackage{array}
\usepackage{colortbl}
\usepackage{subcaption}

\usepackage{arydshln}

\usepackage{tikz}
\usetikzlibrary{positioning, calc, arrows.meta, fit, shadows, decorations.pathmorphing, decorations.pathreplacing, shapes.geometric}
\usepackage{amssymb, bm}
\usepackage{pifont}

\definecolor{obj1}{RGB}{90,90,90}    
\definecolor{obj2}{RGB}{0,130,60}    
\definecolor{obj3}{RGB}{180,100,0}   
\definecolor{obj4}{RGB}{0,80,180}    
\DeclareRobustCommand{\objbadge}[2]{%
  \mbox{\begin{tikzpicture}[baseline=(B.base)]%
    \node[draw=#1, circle, fill=#1!12, text=#1,
      font=\scriptsize\bfseries, inner sep=1.5pt, minimum size=4mm] (B) {#2};%
  \end{tikzpicture}}%
}

\usepackage{hyperref}
\usepackage{url}

\title{AA-SVD: Anchored and Adaptive SVD for Large Language Model Compression}

\author{\name Atul Kumar Sinha \email atul.sinha@unige.ch \\
      University of Geneva, Geneva, Switzerland\\ \\
      \AND
      \name François Fleuret \email francois.fleuret@unige.ch \\
      University of Geneva, Geneva, Switzerland\\
      FAIR, Meta}

\newcommand{\ourmethod}{{\textsc{AA-SVD}}}
\newcommand{\ourmethodlight}{{\textsc{AA-SVD}}}

\def\month{MM}  
\def\year{YYYY} 
\def\openreview{\url{https://openreview.net/forum?id=XXXX}} 

\begin{document}

\maketitle

\begin{abstract}
We introduce a fast low-rank factorization-based framework for compressing large language models that enables rapid compression of billion-parameter models without retraining. 
Unlike existing factorization-based approaches that optimize only on the original inputs, ignoring distribution shifts from upstream compression and thus propagating errors forward, or those that
rely only on shifted inputs and risk drifting away from the original outputs, our approach accounts for both. 
Beyond individual layer compression, we further refine each transformer block end-to-end, minimizing block-level output distortion and allowing compressed layers to jointly compensate for accumulated errors. By anchoring each compressed layer to the original outputs while explicitly modeling input distribution shifts, 
our method finds a low-rank approximation that maintains functional equivalence with the original model.
Experiments on large language models show that our method consistently outperforms existing SVD-based baselines across compression ratios, with the advantage becoming increasingly pronounced at aggressive compression budgets, where competing methods degrade substantially or collapse entirely, offering a practical solution for efficient, large-scale model deployment.\footnote{Project page at \url{https://github.com/atulkumarin/AA-SVD}.}

\end{abstract}

\section{Introduction}
\label{sec:intro}
\begin{wrapfigure}{r}{0.5\textwidth}
    \centering
    \vspace{-1cm}
    \includegraphics[width=0.5\textwidth]{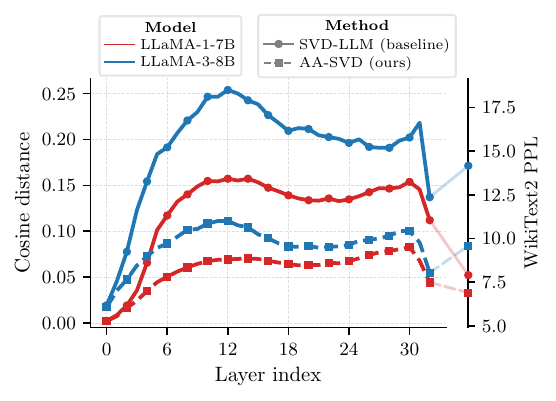}
    \caption{Distortion (cosine distance) between \emph{intermediate} features of the original and compressed model. Diagonal lines link each method's final-layer distortion to its WikiText2 perplexity. \ourmethod~suppresses compression error consistently across depth.}
    \label{fig:main_page_figure}
    \vspace{-1cm}
\end{wrapfigure}
The rapid progress of large-scale pretrained models has fundamentally transformed natural language processing.
Modern large language models (LLMs)~\citep{brown2020language, touvron2023LLaMA, zhang2022opt, achiam2023gpt} now routinely contain tens to hundreds of billions of parameters, enabling remarkable generalization across a wide range of downstream tasks~\citep{kaplan2020scaling}.
However, this scaling has come at steep computational cost: training, fine-tuning, and inference with such models require clusters of high-memory GPUs, making them prohibitively expensive to deploy in resource-constrained or latency-sensitive settings~\citep{patterson2021carbon}.

One promising direction is to move beyond ever-larger models toward smaller, more efficient ones.
Compact models can be trained from scratch for specialized tasks, but this approach sacrifices the broad generalization ability of large pretrained networks.
Alternatively, smaller models can be obtained by \emph{distilling} large networks into student models trained to mimic their behavior~\citep{hinton2015distilling, xu2024survey}, or by applying \emph{post-training compression} techniques such as pruning, quantization, or low-rank factorization~\citep{cheng2017survey, zhu2024survey}.
While both approaches reduce memory footprint and inference cost, distillation typically requires substantial retraining data and compute~\citep{hinton2015distilling, jiao2020tinybert, sanh2019distilbert}, whereas post-training compression can often be applied more rapidly to pretrained networks~\citep{frantar2022gptq, dettmers2022llmint8, wang2025svdllm}, thereby offering a practical path towards democratizing deployment.

A wide range of model compression techniques have been proposed:
\emph{Pruning} removes redundant weights or structures from neural networks, with early work on unstructured sparsification~\citep{han2015learning} and the lottery ticket hypothesis~\citep{frankle2019lottery} showing that smaller subnetworks can be retrained to match dense counterparts.
While effective, pruning often requires iterative retraining and specialized sparsity-aware hardware to fully realize efficiency gains, though recent advances such as SparseGPT and its variants~\citep{frantar2023sparsegpt,ma2023llm, ashkboos2024slicegpt, an2024fluctuation} have enabled post-training pruning of large language models.
\emph{Quantization} reduces numerical precision of weights and activations; modern methods like LLM.int8()~\citep{dettmers2022llmint8}, QLoRA~\citep{dettmers2023qlora}, and AWQ~\citep{lin2024awq} allow near-lossless compression of transformers, though very low-bit settings may require careful calibration.
Another line of work leverages the inherent low-rank structure of network weights: \emph{low-rank factorization} decomposes large matrices into compact representations, reducing both parameters and computation.
Early applications in CNNs~\citep{denton2014exploiting,tai2015convolutional} demonstrated significant speedups, but naïve SVD truncation is known to degrade accuracy.
More recent activation-aware approaches for LLMs~\citep{yuan2023asvd,wang2025svdllm,Li2025AdaSVD,Wang2025DobiSVD} explicitly account for input activations, mitigating this limitation at the cost of additional computation.


These methods differ in their retraining requirements, their dependence on large datasets versus small calibration samples, the efficiency with which compression can be applied to pretrained networks, the degree to which downstream accuracy is preserved, and the extent to which the resulting compressed structure aligns with modern accelerators~\citep{cheng2018survey}.
\emph{SVD-based methods} are especially appealing: they exploit the inherent low-rank structure of neural network weights, yielding compressed models without the need for expensive retraining~\citep{denton2014exploiting,jaderberg2014speeding}.
A straightforward approach is to directly truncate weight matrices by retaining only the top singular components, but this often leads to severe degradation because it treats all input directions equally and discards information that is important for the actual distribution of activations~\citep{denil2013predicting,chen2021drone,wang2025svdllm}.
This limitation has been repeatedly observed in large-scale networks, where naïve low-rank truncation fails to preserve task accuracy and generalization.
To address this, activation-aware approaches have been developed that tailor the factorization to the input distribution, thereby retaining the directions most relevant to the network’s operation.
However, existing activation-aware SVD methods often optimize low-rank approximations using only the original input distribution~\citep{yuan2023asvd,wang2025svdllm,Li2025AdaSVD,Wang2025DobiSVD}, ignoring the shift introduced by upstream compression, which can propagate errors and degrade downstream performance.
Conversely, methods that rely exclusively on shifted inputs, such as Dobi-SVD~\citep{Wang2025DobiSVD}, risk deviating from the original network behavior, introducing instability and loss of fidelity.

In this work, we present \ourmethod, a \emph{fast low-rank factorization-based framework} for compressing pretrained networks.
Our approach accounts for both the original outputs and the distribution shifts caused by upstream compression.
This design yields compressed layers that more faithfully preserve the functional behavior of the uncompressed model, enabling post-training compression of billion-parameter networks without retraining.
Additionally, \ourmethod~refines all compressed layers within a block jointly, minimizing the block-output error and allowing layers to compensate for each other's residual errors.
Figure~\ref{fig:main_page_figure} illustrates how \ourmethod~suppresses compression error consistently across depth compared to prior methods.

\section{Related Work}
\label{sec:related}

Low-rank factorization, e.g., via singular value decomposition (SVD),  has emerged as a promising direction for compressing large pretrained models.
Unlike pruning (irregular sparsity) or quantization (specialized kernels), factorization yields dense, structured factors—enabling the commutation $(\mU\mV^\top)\mX = \mU(\mV^\top\mX)$—that integrate seamlessly with standard linear algebra libraries and reduce both parameters and FLOPs.
Crucially, they can be applied post-training with only a small\footnote{usually 64-1024 samples} calibration set, making them attractive when retraining is infeasible. Recent methods such as ASVD~\citep{yuan2023asvd}, SVD-LLM~\citep{wang2025svdllm}, AdaSVD~\citep{Li2025AdaSVD}, SVD-LLM V2~\citep{wang2025svdllmv2}, Dobi-SVD~\citep{Wang2025DobiSVD}, DipSVD~\citep{ding2025dipsvd}, and SAES-SVD~\citep{hu2026saes} have demonstrated the viability of this approach at scale in large language models.
Based on the optimization objective, compression methods can be broadly grouped into the following categories (Figure~\ref{fig:overview} (left) gives a visual overview):

\begin{figure}[t]
    \centering
    \begin{minipage}[t]{0.59\textwidth}
        \vspace{0pt}
        \centering
        \resizebox{\textwidth}{!}{\input{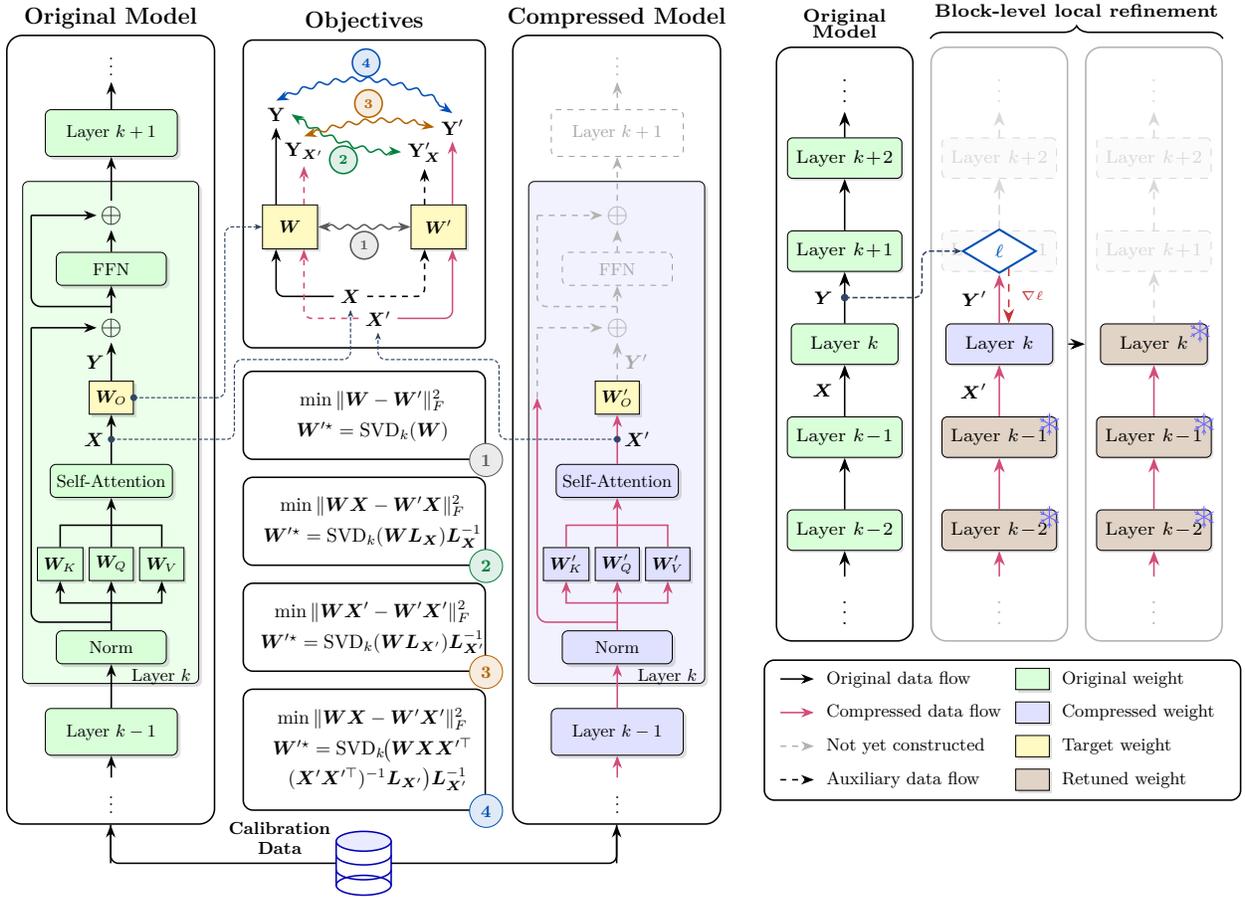}}%
    \end{minipage}%
    \hfill
    \begin{minipage}[t]{0.39\textwidth}
        \vspace{12pt}
        \centering
        \vspace{-0.5cm}
        \resizebox{0.95\textwidth}{!}{

\begin{tikzpicture}[
    >=Stealth,
    node distance=0.5cm,
]

\definecolor{compbrown}{RGB}{180,145,110}   
\definecolor{obj4}{RGB}{0,80,180}            
\definecolor{gradred}{RGB}{200,50,50}        
\definecolor{prov}{RGB}{50,70,100}           

\tikzset{
    seqlayer/.style={
        draw, rectangle, rounded corners=2pt,
        minimum width=1.45cm, minimum height=0.55cm,
        font=\scriptsize, align=center,
        drop shadow={shadow xshift=0.2ex, shadow yshift=-0.2ex, opacity=0.2}
    },
    seqorig/.style={seqlayer, fill=green!15},
    seqcomp/.style={seqlayer, fill=compbrown!40},
    seqstep/.style={seqlayer, fill=blue!12},
    seqghost/.style={
        draw=gray!30, rectangle, rounded corners=2pt, dashed,
        fill=gray!3, minimum width=1.45cm, minimum height=0.55cm,
        font=\scriptsize, align=center, text=gray!30
    },
    seqouter/.style={
        draw, rounded corners=6pt, semithick,
        inner xsep=4pt, inner ysep=3pt, fill=none
    },
    origconn/.style={semithick, ->},                        
    compconn/.style={semithick, ->, draw=purple!70},        
    seqgconn/.style={semithick, ->, draw=gray!30, dashed},  
    provenance/.style={
        draw=prov, semithick, dash pattern=on 2pt off 1pt,
        -{Stealth[length=3pt, width=2.5pt, fill=prov]},
        rounded corners=3pt,
    },
    provsrc/.style={
        circle, draw=prov, fill=prov, inner sep=0pt,
        minimum size=2.5pt,
    },
}

\def\seqD{2.1}   
\def\srow{0.7}   

\def\sxOrig{0}
\pgfmathsetmacro{\sxStepA}{\seqD}
\pgfmathsetmacro{\sxStepB}{2*\seqD}

\begin{scope}[shift={(\sxOrig, 0)}]
    \node[font=\scriptsize] (origdots) at (0, 0) {$\vdots$};
    \node[seqorig, above=0.35cm of origdots] (origkm2) {Layer $k\!-\!2$};
    \node[seqorig, above=\srow of origkm2] (origkm1) {Layer $k\!-\!1$};
    \node[seqorig, above=\srow of origkm1] (origk) {Layer $k$};
    \node[seqorig, above=\srow of origk] (origkp1) {Layer $k\!+\!1$};
    \node[seqorig, above=\srow of origkp1] (origkp2) {Layer $k\!+\!2$};
    \node[font=\scriptsize, above=0.35cm of origkp2] (origdotstop) {$\vdots$};
    \node[font=\scriptsize\bfseries, left=1pt]
        at ($(origkm1.north)!0.5!(origk.south)$) {$\bm{X}$};
    \node[font=\scriptsize\bfseries, left=1pt]
        (origY) at ($(origk.north)!0.5!(origkp1.south)$) {$\bm{Y}$};
    \draw[origconn] (origdots.north) -- (origkm2.south);
    \draw[origconn] (origkm2.north) -- (origkm1.south);
    \draw[origconn] (origkm1.north) -- (origk.south);
    \draw[origconn] (origk.north) -- (origkp1.south);
    \draw[origconn] (origkp1.north) -- (origkp2.south);
    \draw[origconn] (origkp2.north) -- (origdotstop.south);
    \node[seqouter,
          fit=(origdots)(origdotstop)(origkm2)(origkp2)] (origbox) {};
    \node[above=0.0cm of origbox.north, font=\bfseries\scriptsize, align=center]
        {Original\\[-2pt]Model};
\end{scope}

\begin{scope}[shift={(\sxStepA, 0)}]
    \node[font=\scriptsize, text=gray!60] (s1dots) at (0, 0) {$\vdots$};
    \node[seqcomp, above=0.35cm of s1dots] (s1km2) {Layer $k\!-\!2$};
    \node[seqcomp, above=\srow of s1km2] (s1km1) {Layer $k\!-\!1$};
    \node[seqstep, above=\srow of s1km1] (s1k) {Layer $k$};

    \node[seqghost, above=\srow of s1k] (s1kp1) {Layer $k\!+\!1$};
    \node[seqghost, above=\srow of s1kp1] (s1kp2) {Layer $k\!+\!2$};
    \node[font=\scriptsize, text=gray!30, above=0.35cm of s1kp2] (s1dotstop) {$\vdots$};

    \node[font=\scriptsize\bfseries, text=black, left=1pt]
        at ($(s1km1.north)!0.5!(s1k.south)$) {$\bm{X}'$};

    \draw[compconn] (s1dots.north) -- (s1km2.south);
    \draw[compconn] (s1km2.north) -- (s1km1.south);
    \draw[compconn] (s1km1.north) -- (s1k.south);

    \draw[seqgconn] (s1kp1.north) -- (s1kp2.south);
    \draw[seqgconn] (s1kp2.north) -- (s1dotstop.south);

    \node[draw=obj4, thick, diamond, aspect=1.8,
          fill=white, minimum width=1.0cm, minimum height=0.6cm,
          font=\scriptsize\bfseries, text=obj4, inner sep=1pt]
        (lossnode) at (s1kp1) {$\ell$};

    \draw[compconn] (s1k.north) -- (lossnode.south)
        node[midway, left=1pt, font=\scriptsize\bfseries, text=black] {$\bm{Y}'$};

    \draw[seqgconn] (lossnode.north) -- (s1kp2.south);

    \coordinate (gradStart) at ($(lossnode.south)!0.5!(lossnode.south east)$);
    \draw[gradred, semithick, dashed, ->]
        (gradStart) -- (gradStart |- s1k.north)
        node[midway, right=1pt, font=\tiny\bfseries, text=gradred] {$\nabla\ell$};

    \foreach \fnode in {s1km2, s1km1} {
        \begin{scope}[shift={([xshift=-3pt, yshift=-3pt]\fnode.north east)},
                      scale=0.15, draw=blue!60, thin]
            \draw (0,-1) -- (0,1);
            \draw (-0.87,-0.5) -- (0.87,0.5);
            \draw (-0.87,0.5) -- (0.87,-0.5);
            \foreach \a in {0,60,...,300} {
                \draw[rotate=\a] (0,0.55) -- (-0.18,0.8) (0,0.55) -- (0.18,0.8);
            }
        \end{scope}
    }

    \node[seqouter, draw=gray!60,
          fit=(s1dots)(s1dotstop)(s1km2)(s1kp2)] (s1box) {};
\end{scope}

\begin{scope}[shift={(\sxStepB, 0)}]
    \node[font=\scriptsize, text=gray!60] (s2dots) at (0, 0) {$\vdots$};
    \node[seqcomp, above=0.35cm of s2dots] (s2km2) {Layer $k\!-\!2$};
    \node[seqcomp, above=\srow of s2km2] (s2km1) {Layer $k\!-\!1$};
    \node[seqcomp, above=\srow of s2km1] (s2k) {Layer $k$};
    \node[seqghost, above=\srow of s2k] (s2kp1) {Layer $k\!+\!1$};
    \node[seqghost, above=\srow of s2kp1] (s2kp2) {Layer $k\!+\!2$};
    \node[font=\scriptsize, text=gray!30, above=0.35cm of s2kp2] (s2dotstop) {$\vdots$};

    \draw[compconn] (s2dots.north) -- (s2km2.south);
    \draw[compconn] (s2km2.north) -- (s2km1.south);
    \draw[compconn] (s2km1.north) -- (s2k.south);
    \draw[seqgconn] (s2k.north) -- (s2kp1.south);
    \draw[seqgconn] (s2kp1.north) -- (s2kp2.south);
    \draw[seqgconn] (s2kp2.north) -- (s2dotstop.south);

    \foreach \fnode in {s2km2, s2km1, s2k} {
        \begin{scope}[shift={([xshift=-3pt, yshift=-3pt]\fnode.north east)},
                      scale=0.15, draw=blue!60, thin]
            \draw (0,-1) -- (0,1);
            \draw (-0.87,-0.5) -- (0.87,0.5);
            \draw (-0.87,0.5) -- (0.87,-0.5);
            \foreach \a in {0,60,...,300} {
                \draw[rotate=\a] (0,0.55) -- (-0.18,0.8) (0,0.55) -- (0.18,0.8);
            }
        \end{scope}
    }

    \node[seqouter, draw=gray!60,
          fit=(s2dots)(s2dotstop)(s2km2)(s2kp2)] (s2box) {};
\end{scope}

\draw[semithick, ->] (s1box.east) -- (s2box.west);


\coordinate (srcY) at ($(origk.north)!0.5!(origkp1.south)$);
\node[provsrc] at (srcY) {};

\coordinate (provLane) at ($(origbox.east)!0.5!(s1box.west)$);

\draw[provenance]
    (srcY) -- (provLane |- srcY)
    -- (provLane |- lossnode.west) -- (lossnode.west);

\draw[decorate, decoration={brace, amplitude=5pt, raise=3pt}, semithick]
    ([yshift=0pt]s1box.north west) -- ([yshift=0pt]s2box.north east)
    node[midway, above=8pt, font=\scriptsize\bfseries] {Block-level local refinement};

\end{tikzpicture}}%
        \vspace{0.2cm}
        \resizebox{1.\textwidth}{!}{

\begin{tikzpicture}[
    >=Stealth,
]

\definecolor{compbrown}{RGB}{180,145,110}

\def\lrow{0.6}    
\def\llen{0.6}    
\def\ltxt{0.1}    
\def\colsep{4.2}  

\node[draw, rounded corners=5pt, thick, fill=white,
      minimum width=8.5cm, minimum height=2.5cm,
      anchor=north] (legendbox) at (0,0) {};


\coordinate (L0) at ([xshift=0.3cm, yshift=-0.35cm]legendbox.north west);

\begin{scope}[shift={(L0)}]
    \draw[thick, ->] (0, 0) -- +(\llen, 0);
    \node[font=\small, anchor=west] at (\llen+\ltxt, 0) {Original data flow};

    \draw[thick, ->, draw=purple!70] (0, -\lrow) -- +(\llen, 0);
    \node[font=\small, anchor=west] at (\llen+\ltxt, -\lrow) {Compressed data flow};

    \draw[thick, ->, draw=gray!60, dashed] (0, -2*\lrow) -- +(\llen, 0);
    \node[font=\small, anchor=west] at (\llen+\ltxt, -2*\lrow) {Not yet constructed};

    \draw[thick, ->, dashed] (0, -3*\lrow) -- +(\llen, 0);
    \node[font=\small, anchor=west] at (\llen+\ltxt, -3*\lrow) {Auxiliary data flow};
\end{scope}

\begin{scope}[shift={($(L0)+(\colsep,0)$)}]
    \node[draw, rectangle, fill=green!15, minimum width=0.6cm, minimum height=0.38cm,
          drop shadow={shadow xshift=0.3ex, shadow yshift=-0.3ex, opacity=0.25}]
        at (0.3, 0) {};
    \node[font=\small, anchor=west] at (\llen+\ltxt, 0) {Original weight};

    \node[draw, rectangle, fill=blue!12, minimum width=0.6cm, minimum height=0.38cm,
          drop shadow={shadow xshift=0.3ex, shadow yshift=-0.3ex, opacity=0.25}]
        at (0.3, -\lrow) {};
    \node[font=\small, anchor=west] at (\llen+\ltxt, -\lrow) {Compressed weight};

    \node[draw, rectangle, fill=yellow!30, minimum width=0.6cm, minimum height=0.38cm,
          drop shadow={shadow xshift=0.3ex, shadow yshift=-0.3ex, opacity=0.25}]
        at (0.3, -2*\lrow) {};
    \node[font=\small, anchor=west] at (\llen+\ltxt, -2*\lrow) {Target weight};

    \node[draw, rectangle, fill=compbrown!40, minimum width=0.6cm, minimum height=0.38cm,
          drop shadow={shadow xshift=0.3ex, shadow yshift=-0.3ex, opacity=0.25}]
        at (0.3, -3*\lrow) {};
    \node[font=\small, anchor=west] at (\llen+\ltxt, -3*\lrow) {Retuned weight};
\end{scope}

\end{tikzpicture}}%
    \end{minipage}
    \caption[Overview of the two-stage compression pipeline.]{Overview of the two-stage compression pipeline.
             \textbf{Left:}~Four layer-wise compression objectives, differing in which inputs and outputs are compared.
             \objbadge{obj1}{1}~\emph{Input-agnostic}: $\|\mW - \mW'\|_F^2$ — ignores activations entirely.
             \objbadge{obj2}{2}~\emph{Input-aware}: $\|\mW\mX - \mW'\mX\|_F^2$ — matches outputs on original inputs $\mX$.
             \objbadge{obj3}{3}~\emph{Shift-aware}: $\|\mW\mX' - \mW'\mX'\|_F^2$ — matches outputs on the shifted inputs $\mX'$ seen after upstream compression.
             \objbadge{obj4}{4}~\emph{Anchored adaptive} (ours): $\|\mW\mX - \mW'\mX'\|_F^2$ — anchors the target to the original output while conditioning on the shifted input, combining an uncorrupted reference with distribution-shift awareness.
             \textbf{Right:}~Block-level local refinement. Stage~1 factorizes all linear layers in the block independently via any layer-wise objective.
             Stage~2 then jointly optimizes all factorized weights to minimize the block-output error $\ell=\|\mathcal{L}(\mX) - \mathcal{L}'(\mX')\|_F^2$, keeping upstream blocks frozen — the same anchored adaptive spirit as \objbadge{obj4}{4} but applied at block granularity.
             This lets the compressed layers within a block compensate for each other's residual errors, substantially recovering block-output fidelity.}
    \label{fig:overview}
    \vspace{-0.4cm}
\end{figure}
\paragraph{Input-agnostic compression.}
The simplest approach compresses a sub-module $f$ without reference to its inputs, optimizing over the module's parameters alone:
$
\min_{f' \in \gF} \ d \: \!\bigl(f,\, f'\bigr),
$
where $\gF$ denotes the family of admissible compressed sub-modules (e.g., rank-constrained matrices for factorization, sparse masks for pruning), and $d$ is a distance defined purely on the parameters of $f$ and $f'$, with no dependence on any input data.
For pruning, it corresponds to magnitude-based removal of weights or neurons~\citep{han2015learning}; and in
quantization, to rounding without calibration. With low-rank factorization for a linear layer $f(\vx) = \mW\vx$, this takes the form
\[
\min_{\mW' : \mathrm{rank}(\mW') = k} \ \| \mW - \mW' \|_F,
\]
which is solved in closed form by the truncated SVD of $\mW$ using the Eckart--Young theorem,
replacing it by a rank-\(k\) approximation \(\mW'\) constructed from its top singular
components~\citep{Halko2011RandomizedSVD, sainath2013low}.
These methods require no data and are fully order-independent: each sub-module is compressed in isolation with no coupling to the others.
However, they treat all parameter directions uniformly, ignoring the fact that in deep
networks the actual input activations lie in a highly anisotropic
subspace~\citep{ortiz2020neural}: directions preserved by parameter-space
approximations may not align with those that matter for downstream
performance~\citep{chen2021drone, idelbayev2020low}.

\paragraph{Input-aware compression.}
A natural refinement is to account for the geometry of the intermediate features that the sub-module encounters during inference:
$
\min_{f' \in \gF} \ d \: \!\bigl(f(\mX),\, f'(\mX)\bigr),
$
where $\mX \in \mathbb{R}^{n \times l}$ collects intermediate activations at the input of $f$ from the \emph{original, uncompressed} network on calibration samples, and $d$ measures the output discrepancy.
For a linear layer $f(\vx) = \mW\vx$, taking $d$ to be the squared Frobenius norm specializes this to
\[
\min_{\mW' : \mathrm{rank}(\mW') = k} \ \| \mW \mX - \mW' \mX \|_F^2,
\]
a formulation adopted in DRONE~\citep{chen2021drone}, ASVD, SVD-LLM, AdaSVD, SVD-LLM V2 and DipSVD. Pruning methods like FLAP~\citep{an2024fluctuation} use a related objective, leveraging activation statistics from the original network to guide structured pruning decisions.
By preserving the action of $\mW$ on the occupied feature subspace, this is far more faithful to downstream behavior than the input-agnostic objective, and because $\mX$ is fixed, sub-module objectives are fully decoupled and can be compressed in any order.
However, as layers are compressed sequentially, the actual inputs received by each sub-module increasingly diverge from $\mX$ --- and since input-aware methods do not account for this error accumulation, the compressed model's behavior can diverge substantially from the original.

\paragraph{Shift-aware compression.}
A key limitation of input-aware methods is that $\mX$ is produced by the \emph{original} network, whereas in a sequentially compressed pipeline the sub-module actually receives different intermediate features — those produced by the upstream compressed layers.
Shift-aware methods address this by instead minimizing
$
\min_{f' \in \gF} \ d \: \!\bigl(f(\mX'),\, f'(\mX')\bigr),
$
where $\mX' \in \mathbb{R}^{n \times l}$ collects the intermediate features at the input of $f$, produced by running the \emph{partially compressed} network on the same calibration samples, and $d$ measures the output discrepancy on those shifted features.
For a linear layer $f(\vx) = \mW\vx$, taking $d$ to be the squared Frobenius norm gives
\[
\min_{\mW' : \mathrm{rank}(\mW') = k} \ \| \mW \mX' - \mW' \mX' \|_F^2,
\]
as adopted in Dobi-SVD~\citep{Wang2025DobiSVD}, with related ideas in earlier CNN methods~\citep{denton2014exploiting, jaderberg2014speeding} and layer-wise distillation~\citep{jiao2020tinybert}.
The same principle underlies quantization methods such as GPTQ~\citep{frantar2022gptq} and pruning methods such as SparseGPT~\citep{frantar2023sparsegpt}, which process weights in a fixed sequential order, conditioning each update on the outputs of already-compressed predecessors.
By anchoring to the intermediate features the sub-module truly receives, shift-aware methods can mitigate error propagation through the stack.
In stark contrast to input-agnostic and input-aware methods, ordering is not a matter of convenience but a hard requirement: since $\mX'$ depends on all upstream compressed layers, shift-aware compression must follow a valid topological order—compressing out of order yields features $\mX'$ inconsistent with any valid partial compression state.
Their drawback is that when upstream compression has degraded representations, anchoring solely to $\mX'$ risks amplifying divergence from the original network's behavior; moreover, $\mX'$ is estimated from a finite calibration batch and may be noisy, introducing instability.
Thus shift-aware objectives alone provide only a partial solution.

Beyond the choice of approximation objective, the effectiveness of low-rank
factorization depends critically on how ranks are distributed across layers.
Uniform allocation ignores heterogeneity in both compressibility and functional
importance. ASVD~\citep{yuan2023asvd} proposed Sensitivity-based Truncation Rank
Searching (STRS), which evaluates the sensitivity of each linear module to truncation
at different rank levels in isolation, measuring sensitivity as the change in
perplexity on the calibration dataset; this requires repeated full-model evaluations
across modules and rank levels, making it expensive.
SVD-LLM V2~\citep{wang2025svdllmv2} takes a different heuristic approach,
reallocating rank based on the truncation loss \(\|\mW\mX - \mW'\mX\|_F^2\) observed
after an initial uniform compression. Adaptive strategies
such as AdaSVD~\citep{Li2025AdaSVD} leverage layer-importance signals to allocate
more rank where needed, in line with importance-based pruning approaches such as
ShortGPT~\citep{men2024shortgpt}. More principled methods include analytical formulations~\citep{solgi2025activation, abbasi2026zero} and differentiable relaxations~\citep{rausch2025globally, Wang2025DobiSVD} that optimize rank allocation end-to-end, and learned mask approaches~\citep{gao2024adaptive, sundrani2025low, xv2025ara} that select singular components via gradient descent.


\section{\ourmethodlight}
\label{sec:method}

As established in Section~\ref{sec:related}, existing SVD-based compression methods fall into three broad categories that each capture only a partial view of the compression problem: input-agnostic methods ignore the input distribution entirely; input-aware methods account for the original activations but are blind to shifts introduced by upstream compression; and shift-aware methods adapt to the modified inputs but risk drifting from the original network's behavior.
We present \ourmethod~(Anchored and Adaptive SVD), a compression framework that bridges these perspectives.
The central insight is that a faithfully compressed layer must simultaneously satisfy two constraints: its outputs should remain close to those of the uncompressed model, and it must operate correctly on the inputs it will actually receive at inference time—which, after upstream layers have been compressed, may differ substantially from the original activations.
A second insight is that minimizing the error of each linear layer independently is not sufficient: errors across the multiple linear layers within a transformer block can interact, so that even small per-layer errors compound into a larger distortion at the block output.
We therefore introduce a block-level refinement step that minimizes the output error of the entire block after its linear layers have been compressed—allowing the compressed layers to jointly compensate for one another's errors regardless of which layer-wise objective was used.

\subsection{Preliminaries}
\label{sec:prelim}

We consider a pretrained model $\mathcal{M}$ comprising a sequence of $B$ transformer blocks $\{\mathcal{L}_i\}_{i=1}^{B}$, applied sequentially.
Each block $\mathcal{L}_i$ is composed of multiple linear layers—parameterized by weight matrices—together with non-linear operations such as normalization and activations.
Our compression procedure operates at two granularities: at the \emph{linear-layer level}, where each weight matrix $\mW$ within a block is individually approximated by a low-rank matrix; and at the \emph{block level}, where the compressed linear layers within $\mathcal{L}_i$ are jointly refined.

We collect a calibration set of $N$ samples and, for any component $f \in \{\mW, \mathcal{L}_i\}$, denote by $\mX$ the matrix of its input activations on the calibration set (stacked column-wise) and by $f(\mX)$ its corresponding outputs.
For a linear layer, $\mX \in \mathbb{R}^{n \times l}$ and $f(\mX) = \mW\mX \in \mathbb{R}^{m \times l}$; for a block, $\mX \in \mathbb{R}^{d \times l}$ and $f(\mX) = \mathcal{L}_i(\mX) \in \mathbb{R}^{d \times l}$.

When components are compressed sequentially, each receives \emph{shifted} intermediate features produced by upstream compressed components rather than the original network.
We denote by $\mX'$ the corresponding shifted activations—collected by running the same calibration samples through the \emph{partially compressed} network up to (but not including) the current component.
For a linear layer, $f' = \mW'$ is a low-rank matrix with $\mathrm{rank}(\mW') = k \leq \min(m,n)$, decomposed as $\mW' = \mU\mV^\top$ with $\mU \in \mathbb{R}^{m \times k}$ and $\mV \in \mathbb{R}^{n \times k}$.
For a block, $f' = \mathcal{L}'_i$ denotes the block with its linear layers replaced by their low-rank approximations and subsequently refined with a block-level objective.

We now establish the key mathematical results underlying our approach.
We begin with the classical Eckart--Young--Mirsky theorem, which characterizes the optimal low-rank approximation of a matrix in Frobenius norm, and then use it to derive a closed-form solution for the \ourmethodlight~layer-wise objective.

\begin{lemma}[Eckart--Young--Mirsky]
\label{lem:eym}
Let $\mW \in \mathbb{R}^{m \times n}$ with thin SVD $\mW = \mU\mSigma\mV^\top$. Then
\[
\min_{\operatorname{rank}(\mW') \le k} \|\mW - \mW'\|_F^2 = \sum_{i > k} \sigma_i(\mW)^2,
\]
and the unique minimizer is $\mW'^\star = \operatorname{SVD}_k(\mW) = \mU_k\mSigma_k\mV_k^\top$, the truncation to the top-$k$ singular components.
\end{lemma}

\begin{theorem}[]
\label{thm:lowrank-cholesky-unreg}
Let $\mW\in\mathbb{R}^{m\times n}$ be a fixed weight matrix and $\mA,\mB\in\mathbb{R}^{n\times l}$ be any two matrices.
Fix a target rank $k\in\mathbb{N}$. Consider the optimization problem
\begin{equation}
\label{eq:obj-unreg}
\min_{\operatorname{rank}(\mW')\le k}\;\bigl\|\,\mW\mA - \mW'\mB\,\bigr\|_F^2.
\end{equation}
Suppose $\mB\mB^\top$ is invertible, and let $\mL_\mB$ be any invertible matrix satisfying $\mB\mB^\top = \mL_\mB\mL_\mB^\top$\footnote{Such a decomposition can be found using Cholesky decomposition or eigenvalue decomposition.}.
Then an optimal solution to \eqref{eq:obj-unreg} is
\[
\mW'^{\star} \;=\; \operatorname{SVD}_k\!\Bigl(\mW\mA\mB^{\top}\bigl(\mB\mB^{\top}\bigr)^{-1}\mL_\mB\Bigr)\,\mL_\mB^{-1},
\]
where $\operatorname{SVD}_k(\cdot)$ denotes the best rank-$k$ approximation given by Lemma~\ref{lem:eym}.
\end{theorem}

\begin{proof}
See Appendix~\ref{sec:proof-lowrank}.
\end{proof}

\begin{corollary}[No distribution shift]
\label{cor:no-shift}
If $\mB = \mA$, then $\mA\mB^\top = \mB\mB^\top = \mL_\mB\mL_\mB^\top$, so $\mM = \mW\mL_\mB^\top$.
The solution reduces to
$
\mW'^{\star} = \operatorname{SVD}_k\!\bigl(\mW\mL_\mB\bigr)\,\mL_\mB^{-1},
$
the standard whitening-based low-rank regression solution.
\end{corollary}

\subsection{Linear Layer Compression}

Our goal is to compress each linear transformation while ensuring that the resulting network remains \emph{locally faithful} to the original model under the inputs it will actually encounter.
Concretely, for a weight matrix $\mW \in \mathbb{R}^{m \times n}$ with original inputs $\mX \in \mathbb{R}^{n \times l}$ and shifted inputs $\mX' \in \mathbb{R}^{n \times l}$ (after upstream compression), we seek a low-rank approximation $\mW' \in \mathbb{R}^{m \times n}$ that solves
\[
\min_{\mW' : \mathrm{rank}(\mW') = k} \ \| \mW \mX - \mW' \mX' \|_F^2.
\]
This objective enforces that the compressed outputs $\mW'\mX'$ stay close to the original outputs $\mW\mX$, anchoring the compressed network to the behavior of the uncompressed one while simultaneously adapting to the shifted input distribution.
By explicitly constraining $\mathrm{rank}(\mW')=k$, the problem is well-posed as a low-rank regression: we seek the best rank–$k$ approximation of the mapping from $\mX'$ to $\mW\mX$.
This admits a closed-form solution as shown in Theorem~\ref{thm:lowrank-cholesky-unreg}. Figure~\ref{fig:overview} (left) illustrates the per-layer compression stage.

The solution operates only on the covariance matrices $\mX\mX'^{\top}$ and $\mX'\mX'^{\top}$, not on raw activations, so its cost is independent of the number of calibration tokens. We summarize the procedure in Algorithm~\ref{alg:chol-compression} and provide further details in Appendix~\ref{sec:appendix_linear}.

\subsection{Block-Level Local Refinement}

Although each linear layer is compressed to minimize its own output error, the errors introduced by different layers within the same transformer block can interact.
A small residual error at one layer shifts the activations seen by subsequent layers, so that even modest per-layer errors can compound into a larger distortion at the block output (see Figure~\ref{fig:error_evolution}).
To address this, after all linear layers in a block have been compressed we introduce a \emph{block-level local refinement} step.
Concretely, for block $\mathcal{L}_i$ with original calibration inputs $\mX$ and shifted inputs $\mX'$ (received after upstream blocks are compressed), we minimize
\[
\min_{\{\mU_j,\mV_j\},\,\boldsymbol{\theta}_i} \mathbb{E}_{\mX \sim \mathcal{D}_i}\bigl[\|\mathcal{L}_i(\mX) - \mathcal{L}'_i(\mX')\|^2\bigr],
\]
where $\mathcal{D}_i$ denotes the distribution of input activations to block $\mathcal{L}_i$ induced by the calibration data, $\mathcal{L}'_i$ denotes the block with each linear layer $\mW_j$ replaced by its factorized approximation $\mU_j\mV_j^\top$, and $\boldsymbol{\theta}_i$ denotes the remaining trainable parameters of the block (e.g., normalization scales and biases).
The optimization is thus over all factorized weights and block-local parameters jointly.
This allows the compressed layers within a block to collectively compensate for one another's residual errors.
The objective is minimized via gradient-based optimization.
Because the refinement is confined to a single block and uses only a small calibration set, it adds negligible overhead while substantially recovering block-output fidelity. Figure~\ref{fig:overview} (right) illustrates the block-level refinement stage.

\begin{algorithm}[t]
\caption{\textsc{CompressLayer}: \ourmethod~layer-wise low-rank compression}
\label{alg:chol-compression}
\begin{algorithmic}[1]
\REQUIRE Weight matrix $\mW\in\mathbb{R}^{m\times n}$,
original inputs $\mX\in\mathbb{R}^{n\times l}$, shifted inputs $\mX'\in\mathbb{R}^{n\times l}$, target rank $k$
\STATE Set $\mA = \mX,\,\mB = \mX'$ \hfill\COMMENT{shift-aware: $\mA\!=\!\mB\!=\!\mX'$;\; input-aware: $\mA\!=\!\mB\!=\!\mX$}
\STATE Compute $\mC = \mA\mB^{\top}$ and $\mS = \mB\mB^{\top}$
\STATE Factorize: $\mS = \mR\mR^{\top}$ \hfill\COMMENT{e.g.\ Cholesky or EVD}
\STATE Compute $\mM = \mW\mC\mS^{-1}\mR$
\STATE Truncated SVD: $[\mU_k,\mSigma_k,\mV_k] = \operatorname{SVD}_k(\mM)$
\RETURN factorized weight $\mU = \mU_k\mSigma_k,\;\mV = \mR^{-\top}\mV_k$, so that $\mW' = \mU\mV^\top$
\end{algorithmic}
\end{algorithm}

\begin{algorithm}[t]
\caption{\ourmethod: end-to-end block-wise compression with local refinement}
\label{alg:aa-svd-full}
\begin{algorithmic}[1]
\REQUIRE Model $\mathcal{M}$ with blocks $\{\mathcal{L}_i\}_{i=1}^{B}$, calibration data, target rank $k$
\STATE Extract input activations $\mX \leftarrow \mX' \leftarrow \mathcal{E}(\text{calibration data})$ from the embedding layer $\mathcal{E}$
\FOR{each block $\mathcal{L}_i$ in $\mathcal{M}$}
    \STATE Initialize compressed block $\mathcal{L}'_i \leftarrow \mathcal{L}_i$
    \FOR{each linear layer $\mW_j$ in $\mathcal{L}'_i$}
        \STATE Collect $\mX_j$ from $\mathcal{L}_i$ and $\mX'_j$ from $\mathcal{L}'_i$ by forward pass up to layer $j$
        \STATE $[\mU_j, \mV_j] \leftarrow \textsc{CompressLayer}(\mW_j,\,\mX_j,\,\mX'_j,\,k)$
        \STATE Update $\mW_j \leftarrow \mU_j\mV_j^\top$ in $\mathcal{L}'_i$
    \ENDFOR
    \STATE \textbf{Block-level refinement:} optimize $\{\mU_j,\mV_j\}$ and block-local parameters $\boldsymbol{\theta}_i$ (e.g.\ norms, biases) jointly to minimize $\mathrm{MSE}\!\left(\mathcal{L}_i(\mX),\,\mathcal{L}'_i(\mX')\right)$
    \STATE Update inputs for next block: $\mX \leftarrow \mathcal{L}_i(\mX)$,\quad $\mX' \leftarrow \mathcal{L}'_i(\mX')$
\ENDFOR
\RETURN compressed model $\mathcal{M}'$ with blocks $\{\mathcal{L}'_i\}_{i=1}^{B}$
\end{algorithmic}
\end{algorithm}

The complete end-to-end compression procedure is described in Algorithm~\ref{alg:aa-svd-full}, which processes the model block by block: within each block, \textsc{CompressLayer} (Algorithm~\ref{alg:chol-compression}) is applied to each linear layer in sequence, after which the block-level refinement step is performed before moving to the next block.
Further implementation details are provided in Appendix~\ref{sec:appendix_block}.

\begin{table*}[t]
\centering
\caption{Comparison of \ourmethod~with SOTA methods for SVD-based compression of LLaMA-7B on three language modeling tasks and seven commonsense reasoning benchmarks (zero-shot evaluation) under varying compression ratios. Best performance is marked in bold. $(^\dagger)$ uses LoRA fine-tuning, $(^\ddagger)$ uses dynamic or non-uniform capacity/rank allocation, and $(^q)$ indicates results with Dobi-SVD-style remapping enabled. Results for baseline methods are taken from the original papers or prior work where available.
}
\label{tab:LLaMA-1-7B}
\resizebox{\textwidth}{!}{
\begin{tabular}{c l ccc !{\hspace{3pt}\vline\hspace{3pt}} ccccccccc}
\toprule
\multirow{2}{*}{\textbf{Ratio}} & \multirow{2}{*}{\textbf{Method}} & \multicolumn{3}{c}{\textbf{PPL} ($\downarrow$)} & \multicolumn{9}{c}{\textbf{Accuracy} ($\uparrow$)} \\
\cmidrule(lr){3-5} \cmidrule(lr){6-14}
 & & \textbf{Wiki2} & \textbf{PTB} & \textbf{C4} & \textbf{Openb.} & \textbf{ARC\_e} & \textbf{ARC\_c} & \textbf{WinoG.} & \textbf{PIQA} & \textbf{MathQA} & \textbf{HellaS.} & \textbf{Avg.} & \textbf{Drop (\%)} \\
\specialrule{\lightrulewidth}{2pt}{2pt}
$1.0$ & Dense & $5.68$ & $8.34$ & $7.34$ & $0.34$ & $0.75$ & $0.42$ & $0.69$ & $0.79$ & $0.27$ & $0.57$ & $0.55$ & $-$ \\
\specialrule{\lightrulewidth}{2pt}{2pt}
\multirow{5}{*}{$0.8$}
& ASVD & $11.14$ & $16.55$ & $15.93$ & $0.25$ & $0.53$ & $0.27$ & $0.64$ & $0.68$ & $0.24$ & $0.41$ & $0.43$ & $21.1\%$ \\
& SVD-LLM$^\dagger$ & $7.94$ & $16.22$ & $15.84$ & $0.22$ & $0.58$ & $0.29$ & $0.63$ & $0.69$ & $0.24$ & $0.43$ & $0.44$ & $19.6\%$ \\
& {Dobi-SVD$^\ddagger$} & {$8.54$} & $14.83$ & $\mathbf{10.01}$ & $0.26$ & $0.59$ & $0.31$ & $\mathbf{0.66}$ & $0.70$ & {$0.23$} & $0.44$ & $0.46$ & $16.7\%$ \\
& Dip-SVD$^\ddagger$ & $7.95$ & $15.60$ & $14.07$ & $0.27$ & $0.63$ & $0.33$ & $0.64$ & $0.71$ & $0.24$ & $0.45$ & $0.47$ & $14.6\%$ \\
& SAES-SVD & $7.17$ & $15.16$ & $13.77$ & $0.29$ & $0.68$ & $\mathbf{0.36}$ & $0.65$ & $\mathbf{0.75}$ & $0.25$ & $0.45$ & $0.49$ & $10.4\%$ \\
\rowcolor{gray!12} \cellcolor{white} & \ourmethodlight~& $\mathbf{6.89}$ & $\mathbf{12.30}$ & $12.04$ & $\mathbf{0.31}$ & $\mathbf{0.71}$ & $\mathbf{0.36}$ & $\mathbf{0.66}$ & $0.72$ & $\mathbf{0.25}$ & $\mathbf{0.48}$ & $\mathbf{0.50}$ & $\mathbf{8.9\%}$ \\
\arrayrulecolor{gray!75}\cmidrule(lr){2-14}
\arrayrulecolor{black}
& {Dobi-SVD$^{\ddagger, q}$} & {$6.08$} & $15.39$ & $\mathbf{7.83}$ & $0.27$ & $0.65$ & $0.37$ & $0.68$ & $0.77$ & {$\mathbf{0.27}$} & $\mathbf{0.54}$ & $0.51$ & $7.3\%$ \\
\rowcolor{gray!12} \cellcolor{white} & \ourmethodlight$^q$ & $\mathbf{6.01}$ & $\mathbf{8.97}$ & $8.37$ & $\mathbf{0.30}$ & $\mathbf{0.74}$ & $\mathbf{0.41}$ & $\mathbf{0.69}$ & $\mathbf{0.77}$ & $0.26$ & $0.53$ & $\mathbf{0.53}$ & $\mathbf{3.4\%}$ \\
\specialrule{\lightrulewidth}{2pt}{2pt}
\multirow{5}{*}{$0.6$}
& ASVD & $1407$ & $3292$ & $1109$ & $0.13$ & $0.28$ & $0.22$ & $0.48$ & $0.55$ & $0.19$ & $0.26$ & $0.30$ & $44.9\%$ \\
& SVD-LLM$^\dagger$ & $13.11$ & $63.75$ & $49.83$ & $0.19$ & $0.42$ & $0.25$ & $0.58$ & $0.60$ & $0.21$ & $0.33$ & $0.37$ & $32.6\%$ \\
& {Dobi-SVD$^\ddagger$} & {$13.54$} & ${46.38}$ & $23.54$ & $0.22$ & {$0.41$} & $0.27$ & ${0.58}$ & $0.61$ & {$0.23$} & $0.34$ & $0.38$ & $30.5\%$ \\
& Dip-SVD$^\ddagger$ & $12.76$ & $46.95$ & $34.35$ & $0.22$ & $0.50$ & $0.30$ & $0.61$ & $0.64$ & $0.22$ & $0.36$ & $0.41$ & $25.6\%$ \\
& SAES-SVD & $10.42$ & $45.13$ & $32.79$ & $0.23$ & $0.50$ & $0.29$ & $0.62$ & $\mathbf{0.65}$ & $0.23$ & $0.36$ & $0.41$ & $24.8\%$ \\
\rowcolor{gray!12} \cellcolor{white} & \ourmethodlight~& $\mathbf{8.35}$ & $\mathbf{24.94}$ & $\mathbf{18.97}$ & $\mathbf{0.26}$ & $\mathbf{0.62}$ & $\mathbf{0.31}$ & $\mathbf{0.62}$ & $\mathbf{0.65}$ & $\mathbf{0.23}$ & $\mathbf{0.41}$ & $\mathbf{0.44}$ & $\mathbf{19.1\%}$ \\
\arrayrulecolor{gray!75}\cmidrule(lr){2-14}
& {Dobi-SVD$^{\ddagger, q}$} & {$8.12$} & $43.85$ & $12.63$ & $0.28$ & $0.65$ & $0.32$ & $0.62$ & $0.72$ & {$0.25$} & $0.45$ & $0.47$ & $14.1\%$ \\
\rowcolor{gray!12} \cellcolor{white} & \ourmethodlight$^q$ & $\mathbf{7.09}$ & $\mathbf{11.07}$ & $\mathbf{11.25}$ & $\mathbf{0.28}$ & $\mathbf{0.71}$ & $\mathbf{0.37}$ & $\mathbf{0.65}$ & $\mathbf{0.73}$ & $\mathbf{0.26}$ & $\mathbf{0.49}$ & $\mathbf{0.50}$ & $\mathbf{8.9\%}$ \\
\specialrule{\lightrulewidth}{2pt}{2pt}
\multirow{5}{*}{$0.4$}
& ASVD & $57057$ & $45218$ & $43036$ & $0.12$ & $0.26$ & $0.21$ & $0.49$ & $0.53$ & $0.18$ & $0.26$ & $0.29$ & $46.5\%$ \\
& SVD-LLM$^\dagger$ & $53.74$ & $438.58$ & $383.07$ & $0.14$ & $0.28$ & $0.22$ & $0.50$ & ${0.55}$ & $0.21$ & $0.27$ & $0.31$ & $43.3\%$ \\
& {Dobi-SVD$^\ddagger$} & {$46.18$} & {$238.91$} & $190.62$ & {$0.15$} & {$0.31$} & {$0.20$} & {${0.52}$} & ${0.54}$ & {$0.22$} & $0.28$ & $0.32$ & $42.0\%$ \\
& SAES-SVD & $22.01$ & $116.83$ & $93.97$ & $0.16$ & $0.33$ & $\mathbf{0.25}$ & $0.52$ & $0.54$ & $\mathbf{0.23}$ & $0.30$ & $0.33$ & $39.2\%$ \\
\rowcolor{gray!12} \cellcolor{white} & \ourmethodlight~& $\mathbf{13.67}$ & $\mathbf{74.64}$ & $\mathbf{46.14}$ & $\mathbf{0.19}$ & $\mathbf{0.44}$ & $0.23$ & $\mathbf{0.55}$ & $\mathbf{0.60}$ & $\mathbf{0.23}$ & $\mathbf{0.32}$ & $\mathbf{0.37}$ & $\mathbf{33.2\%}$ \\
\arrayrulecolor{gray!75}\cmidrule(lr){2-14}
& {Dobi-SVD$^{\ddagger, q}$} & {$9.95$} & $67.62$ & $\mathbf{17.94}$ & $0.23$ & $0.52$ & $0.24$ & $0.56$ & $\mathbf{0.65}$ & {$0.23$} & $0.38$ & $0.40$ & $26.6\%$ \\
\rowcolor{gray!12} \cellcolor{white} & \ourmethodlight$^q$ & $\mathbf{8.61}$ & $\mathbf{24.44}$ & ${19.69}$ & $\mathbf{0.26}$ & $\mathbf{0.58}$ & $\mathbf{0.31}$ & $\mathbf{0.62}$ & ${0.64}$ & $\mathbf{0.23}$ & $\mathbf{0.41}$ & $\mathbf{0.44}$ & $\mathbf{20.4\%}$ \\
\bottomrule
\end{tabular}
}
\end{table*}

\section{Experiments}
\label{sec:exp}

We evaluate \ourmethod\ across a diverse set of open-source pretrained language models, spanning multiple architecture families and parameter scales.
Concretely, we compress models from the LLaMA~\citep{touvron2023LLaMA} and Qwen~\citep{bai2023qwen} families, which together cover a broad range of model sizes and training recipes representative of the current landscape.
For calibration, we follow prior work and use 256 samples drawn from the WikiText2~\citep{merity2016pointer} training split unless otherwise stated; our ablations show this modest budget is sufficient for stable compression.
Compressed models are then evaluated along two axes: \emph{language modeling perplexity} on WikiText2, C4~\citep{raffel2020exploring}, and PTB~\citep{marcinkiewicz1994building}, which measures how well the model preserves distributional fidelity; and \emph{zero-shot accuracy on commonsense reasoning} benchmarks --- Winogrande~\citep{sakaguchi2021winogrande}, PIQA~\citep{bisk2020piqa}, ARC-Easy and ARC-Challenge~\citep{clark2018think}, OpenBookQA~\citep{mihaylov2018openbookqa}, HellaSwag~\citep{zellers2019hellaswag} and MathQA~\citep{amini2019mathqa} --- which captures practical downstream utility. 

\subsection{Main Results}

Table~\ref{tab:LLaMA-1-7B} presents a detailed comparison on LLaMA-7B against five SVD-based baselines---ASVD, SVD-LLM, Dobi-SVD, Dip-SVD, and SAES-SVD---across three perplexity benchmarks and seven zero-shot commonsense reasoning tasks at compression ratios of $0.8$, $0.6$, and $0.4$.
Table~\ref{tab:multiple_models_summary} reports aggregated results across five additional models spanning the LLaMA-2, LLaMA-3, and Qwen-2.5 families; expanded per-benchmark breakdowns are provided in Appendix~\ref{sec:comparison_more_models}. We also include results with Dobi-SVD-style remapping enabled for both Dobi-SVD and \ourmethodlight~for a fair comparison; more details on remapping are provided in Appendix~\ref{sec:appendix_dobi_svd_remapping}.

At ratio $0.8$, \ourmethodlight~achieves the best perplexity and average accuracy among all methods without weight remapping, with the nearest competitor (SAES-SVD) incurring a notably larger accuracy drop; enabling weight remapping (\ourmethodlight$^q$) further reduces the accuracy gap to only $3.4\%$, outperforming Dobi-SVD$^{\ddagger,q}$ on both metrics despite Dobi-SVD employing dynamic rank allocation.
As compression becomes more aggressive the margin widens: at ratio $0.6$, \ourmethodlight~reduces perplexity substantially across all three benchmarks while matching or exceeding SAES-SVD on every reasoning task, and with remapping, out-of-domain perplexity (PTB) improves by a particularly large factor over Dobi-SVD$^{\ddagger,q}$.
At ratio $0.4$, ASVD and SVD-LLM become essentially degenerate, while \ourmethodlight~continues to produce functional compressed models, reducing perplexity by nearly $40\%$ relative to SAES-SVD and cutting the accuracy drop by roughly six points.

The gains generalize broadly across architectures (Table~\ref{tab:multiple_models_summary}).
\ourmethodlight~outperforms SVD-LLM on every model family at both evaluated ratios, with the largest gap on LLaMA-3-1B, where SVD-LLM's perplexity degrades by a factor of three---suggesting compact modern architectures are especially sensitive to per-layer approximation error and benefit most from block-level joint optimization.
At ratio $0.6$, SVD-LLM collapses on both LLaMA-3 models, while \ourmethodlight~retains functional representations throughout.
These results consistently demonstrate state-of-the-art performance across ratios, metrics, and model families, with gains most pronounced precisely where competing methods fail---underscoring the importance of minimizing block-level output error rather than compressing each layer in isolation.
\begin{table*}[t]
\centering
\caption{Comparison of \ourmethod~with SOTA methods across multiple models at compression ratios $0.8$ and $0.6$. PPL refers to WikiText2 perplexity; Accuracy is averaged over seven commonsense reasoning benchmarks (zero-shot). Best performance is marked in bold.
}
\label{tab:multiple_models_summary}
\resizebox{\textwidth}{!}{
\begin{tabular}{c l !{\hspace{3pt}\vline\hspace{3pt}} cc !{\hspace{3pt}\vline\hspace{3pt}} cc !{\hspace{3pt}\vline\hspace{3pt}} cc !{\hspace{3pt}\vline\hspace{3pt}} cc !{\hspace{3pt}\vline\hspace{3pt}} cc}
\toprule
\multirow{2}{*}{\textbf{Ratio}} & \multirow{2}{*}{\textbf{Method}}
  & \multicolumn{2}{c}{\textbf{LLaMA-2-7B}}
  & \multicolumn{2}{c}{\textbf{LLaMA-2-13B}}
  & \multicolumn{2}{c}{\textbf{LLaMA-3-1B}}
  & \multicolumn{2}{c}{\textbf{LLaMA-3-8B}}
  & \multicolumn{2}{c}{\textbf{Qwen-2.5-7B}} \\
\cmidrule(lr){3-4} \cmidrule(lr){5-6} \cmidrule(lr){7-8} \cmidrule(lr){9-10} \cmidrule(lr){11-12}
  & & \textbf{PPL} ($\downarrow$) & \textbf{Acc.} ($\uparrow$)
  & \textbf{PPL} ($\downarrow$) & \textbf{Acc.} ($\uparrow$)
  & \textbf{PPL} ($\downarrow$) & \textbf{Acc.} ($\uparrow$)
  & \textbf{PPL} ($\downarrow$) & \textbf{Acc.} ($\uparrow$)
  & \textbf{PPL} ($\downarrow$) & \textbf{Acc.} ($\uparrow$) \\
\specialrule{\lightrulewidth}{2pt}{2pt}
$1.0$ & Baseline
  & $5.47$ & $0.55$
  & $4.88$ & $0.58$
  & $9.75$ & $0.48$
  & $6.24$ & $0.60$
  & $6.84$ & $0.60$ \\
\specialrule{\lightrulewidth}{2pt}{2pt}
\multirow{1}{*}{$0.8$}
& SVD-LLM
  & $8.41$ & $0.43$
  & $6.65$ & $0.48$
  & $45.62$ & $0.32$
  & $14.16$ & $0.44$
  & $10.69$ & $0.47$ \\
\rowcolor{gray!12} \cellcolor{white} & \ourmethodlight
  & $\mathbf{6.84}$ & $\mathbf{0.50}$
  & $\mathbf{5.95}$ & $\mathbf{0.53}$
  & $\mathbf{15.12}$ & $\mathbf{0.39}$
  & $\mathbf{9.58}$ & $\mathbf{0.50}$
  & $\mathbf{8.53}$ & $\mathbf{0.53}$ \\
\specialrule{\lightrulewidth}{2pt}{2pt}
\multirow{1}{*}{$0.6$}
& SVD-LLM
  & $16.47$ & $0.35$
  & $10.79$ & $0.38$
  & $402.76$ & $0.30$
  & $76.31$ & $0.32$
  & $28.67$ & $0.33$ \\
\rowcolor{gray!12} \cellcolor{white} & \ourmethodlight
  & $\mathbf{8.55}$ & $\mathbf{0.44}$
  & $\mathbf{7.44}$ & $\mathbf{0.46}$
  & $\mathbf{23.74}$ & $\mathbf{0.35}$
  & $\mathbf{13.66}$ & $\mathbf{0.41}$
  & $\mathbf{11.00}$ & $\mathbf{0.44}$ \\
\bottomrule
\end{tabular}
}
\end{table*}
\subsection{Comparison with pruning methods}
Table~\ref{tab:pruning_comparison_1} compares zero-shot accuracy on LLaMA-2-7B against four structured pruning methods---LLM-Pruner~\citep{ma2023llm}, SliceGPT~\citep{ashkboos2024slicegpt}, Bonsai~\citep{kolawole2024everybody}, and Wanda-sp~\citep{sun2023simple}---at ratios $0.6$ and $0.5$, and Table~\ref{tab:pruning_comparison_2} reports WikiText2 perplexity on LLaMA-7B under fixed GPU memory budgets and compares \ourmethod~against LLM-Pruner, SliceGPT and BlockPruner~\citep{zhong2025blockpruner}. Together, they situate \ourmethodlight~relative to methods that remove entire model components and therefore benefit from dense-kernel efficiency at inference time.
Without remapping, \ourmethodlight~is competitive with the best pruning methods at ratio $0.6$ (only a $19.8\%$ accuracy drop vs.\ $18.3\%$ for Bonsai), a notable result given that SVD-LLM lags substantially behind all pruning baselines at the same ratio ($37.5\%$ drop); with remapping, \ourmethodlight$^q$ surpasses every pruning method by a clear margin, achieving a $7.1\%$ accuracy drop at ratio $0.6$---less than half that of Bonsai---and $20.7\%$ at ratio $0.4$, competitive with Bonsai's performance at the less aggressive setting.
The memory-budget comparison tells a similar story: \ourmethodlight~achieves the lowest perplexity at every budget from $10$GB down to $7$GB, and the advantage over pruning methods grows as the budget tightens, with structured pruning baselines deteriorating far more sharply under stricter constraints.

\begin{table*}[t]
\centering
\caption{Comparison of \ourmethod~with structured pruning methods on compression performance of LLaMA-2-7B across five commonsense reasoning benchmarks (zero-shot evaluation). Results for baseline methods are taken from \citet{Wang2025DobiSVD}.\\
}
\label{tab:pruning_comparison_1}
\resizebox{0.7\textwidth}{!}{
\begin{tabular}{c l !{\hspace{3pt}\vline\hspace{3pt}} ccccccc}
\toprule
\multirow{2}{*}{\textbf{Ratio}} & \multirow{2}{*}{\textbf{Method}} & \multicolumn{7}{c}{\textbf{Accuracy} ($\uparrow$)} \\
\cmidrule(lr){3-9}
 & & \textbf{PIQA} & \textbf{HellaS.} & \textbf{WinoG.} & \textbf{ARC\_e} & \textbf{ARC\_c} & \textbf{Avg.} & \textbf{Drop (\%)} \\
\specialrule{\lightrulewidth}{2pt}{2pt}
$1.0$ & Dense & $0.78$ & $0.57$ & $0.69$ & $0.76$ & $0.43$ & $0.65$ & $-$ \\
\specialrule{\lightrulewidth}{2pt}{2pt}
\multirow{8}{*}{$0.6$}
& LLM-Pruner  & $0.70$ & $0.41$ & $0.53$ & $0.53$ & $0.27$ & $0.48$ & $24.5\%$ \\
& SliceGPT    & $0.65$ & $0.57$ & $0.60$ & $0.43$ & $0.32$ & $0.51$ & $20.4\%$ \\
& Bonsai      & $0.72$ & $0.45$ & $0.58$ & $0.59$ & $0.30$ & $0.53$ & $18.3\%$ \\
& Wanda-sp    & $0.70$ & $0.42$ & $0.53$ & $0.57$ & $0.29$ & $0.50$ & $22.3\%$ \\
\arrayrulecolor{gray!75}\cmidrule(lr){2-9}
\arrayrulecolor{black}
& SVD-LLM    & $0.58$ & $0.31$ & $0.53$ & $0.39$ & $0.21$ & $0.40$ & $37.5\%$ \\
\rowcolor{gray!12} \cellcolor{white} & \ourmethodlight& ${0.66}$ & ${0.41}$ & ${0.62}$ & ${0.60}$ & ${0.30}$ & ${0.52}$ & ${19.8\%}$ \\
\arrayrulecolor{gray!75}\cmidrule(lr){2-9}
\arrayrulecolor{black}
& {Dobi-SVD$^{\ddagger, q}$}    & $0.72$ & $0.45$ & $0.64$ & $0.67$ & $0.31$ & $0.56$ & $13.6\%$ \\
\rowcolor{gray!12} \cellcolor{white} & \ourmethodlight$^q$& ${0.73}$ & ${0.50}$ & ${0.66}$ & ${0.72}$ & ${0.39}$ & ${0.60}$ & ${7.1\%}$ \\
\specialrule{\lightrulewidth}{2pt}{2pt}
\multirow{4}{*}{$0.5$}
& LLM-Pruner & $0.67$ & $0.35$ & $0.52$ & $0.48$ & $0.22$ & $0.45$ & $30.7\%$ \\
& SliceGPT & $0.58$ & $0.46$ & $0.55$ & $0.37$ & $0.28$ & $0.45$ & $30.7\%$ \\
& Bonsai & $0.66$ & $0.40$ & $0.54$ & $0.49$ & $0.26$ & $0.47$ & $27.2\%$ \\
& Wanda-sp & $0.63$ & $0.32$ & $0.53$ & $0.43$ & $0.20$ & $0.42$ & $34.7\%$ \\
\specialrule{\lightrulewidth}{2pt}{2pt}
& SVD-LLM    & $0.53$ & $0.27$ & $0.49$ & $0.27$ & $0.22$ & $0.36$ & $44.9\%$ \\
\rowcolor{gray!12} \cellcolor{white} $0.4$ & \ourmethodlight& ${0.60}$ & ${0.32}$ & ${0.56}$ & ${0.44}$ & ${0.24}$ & ${0.43}$ & ${33.1\%}$ \\
\arrayrulecolor{gray!75}\cmidrule(lr){2-9}
\arrayrulecolor{black}
& {Dobi-SVD$^{\ddagger, q}$}    & $0.67$ & $0.38$ & $0.57$ & $0.55$ & $0.26$ & $0.49$ & $24.8\%$ \\
\rowcolor{gray!12} \cellcolor{white} & \ourmethodlight$^q$& ${0.65}$ & ${0.40}$ & ${0.61}$ & ${0.60}$ & ${0.30}$ & ${0.51}$ & ${20.7\%}$ \\
\bottomrule
\end{tabular}
}
\end{table*}

\begin{table}[t]
\centering
\caption{Perplexity (WikiText2, $\downarrow$) comparison of \ourmethod~and structured pruning baselines on LLaMA-7B under different memory budgets. Results for baseline methods are taken from \citet{hu2026saes}.}
\label{tab:pruning_comparison_2}
\resizebox{0.82\textwidth}{!}{
\begin{tabular}{c ccccc}
\toprule
\textbf{Memory} & \textbf{LLM-Pruner} & \textbf{SliceGPT} & \textbf{BlockPruner} & \textbf{SAES-SVD} & \textbf{\ourmethod~(Ours)} \\
\midrule
10GB & $9.88$ & $8.78$ & $9.40$ & ${7.17}$ & $\mathbf{6.89}$ \\
9GB  & $12.21$ & $12.73$ & $12.76$ & ${8.22}$ & $\mathbf{7.14}$ \\
8GB  & $18.94$ & $16.39$ & $19.78$ & ${8.96}$ & $\mathbf{7.84}$ \\
7GB  & $21.68$ & $27.41$ & $43.05$ & ${10.15}$ & $\mathbf{8.35}$  \\
\bottomrule
\end{tabular}
}
\end{table}


\subsection{Ablations and Analysis}

\begin{wraptable}{r}{0.5\textwidth}
\centering
\caption{WikiText2 Perplexity and average accuracy over seven commonsense reasoning tasks for \ourmethodlight~with different layer-wise objectives and block-level refinement. Best performance is marked in bold.
}
\label{tab:ablation_objective}
\resizebox{\linewidth}{!}{%
\begin{tabular}{c c c !{\hspace{3pt}\vline\hspace{3pt}} >{\centering\arraybackslash}p{1.5cm} >{\centering\arraybackslash}p{1.5cm}} 
\toprule
\multirow{1}{*}{\textbf{Ratio}} & \multirow{1}{*}{\textbf{Objective}} & \multirow{1}{*}{\textbf{Refinement}} & \multirow{1}{*}{\textbf{PPL} ($\downarrow$)} & \multirow{1}{*}{\textbf{Acc.} ($\uparrow$)} \\
\specialrule{\lightrulewidth}{2pt}{2pt}
$1.0$ & Dense & $-$
  & $5.68$ & $0.55$ \\
\specialrule{\lightrulewidth}{2pt}{2pt}
\multirow{1}{*}{$0.8$}
  & $\|\bm{W} - \bm{W}'\|_F^2$ & \ding{55}
  & \multicolumn{1}{l}{$2e4$} & \multicolumn{1}{l}{$0.31$} \\
  &  & \ding{51}
  & \multicolumn{1}{c}{$7.35$} & \multicolumn{1}{c}{$0.50$} \\
  \arrayrulecolor{gray!75}\cmidrule(lr){2-5}
  \arrayrulecolor{black}
  & $\|\bm{W}\mX - \bm{W}'\mX\|_F^2$ & \ding{55}
  & \multicolumn{1}{l}{$7.89$} & \multicolumn{1}{l}{$0.45$} \\
  &  & \ding{51}
  & \multicolumn{1}{c}{$\mathbf{6.89}$} & \multicolumn{1}{c}{$\mathbf{0.50}$} \\
  \arrayrulecolor{gray!75}\cmidrule(lr){2-5}
  \arrayrulecolor{black}
  & $\|\bm{W}\mX' - \bm{W}'\mX'\|_F^2$ & \ding{55}
  & \multicolumn{1}{l}{$8.22$} & \multicolumn{1}{l}{$0.45$} \\
  &  & \ding{51}
  & \multicolumn{1}{c}{$7.28$} & \multicolumn{1}{c}{$0.45$} \\
  \arrayrulecolor{gray!75}\cmidrule(lr){2-5}
  \arrayrulecolor{black}
  & $\|\bm{W}\mX - \bm{W}'\mX'\|_F^2$ & \ding{55}
  & \multicolumn{1}{l}{$\mathbf{7.68}$} & \multicolumn{1}{l}{$\mathbf{0.46}$} \\
  &  & \ding{51}
  & \multicolumn{1}{c}{$7.08$} & \multicolumn{1}{c}{$0.48$} \\
\specialrule{\lightrulewidth}{2pt}{2pt}
\multirow{1}{*}{$0.6$}
  & $\|\bm{W} - \bm{W}'\|_F^2$ & \ding{55}
  & \multicolumn{1}{l}{$5e5$} & \multicolumn{1}{l}{$0.30$} \\
  & & \ding{51}
  & \multicolumn{1}{c}{$10.93$} & \multicolumn{1}{c}{$\mathbf{0.45}$} \\
  \arrayrulecolor{gray!75}\cmidrule(lr){2-5}
  \arrayrulecolor{black}
  & $\|\bm{W}\mX - \bm{W}'\mX\|_F^2$ & \ding{55}
  & \multicolumn{1}{l}{$13.11$} & \multicolumn{1}{l}{$0.37$} \\
  & & \ding{51}
  & \multicolumn{1}{c}{$\mathbf{8.35}$} & \multicolumn{1}{c}{$0.44$} \\
  \arrayrulecolor{gray!75}\cmidrule(lr){2-5}
  \arrayrulecolor{black}
  & $\|\bm{W}\mX' - \bm{W}'\mX'\|_F^2$ & \ding{55}
  & \multicolumn{1}{l}{$14.87$} & \multicolumn{1}{l}{$0.36$} \\
  &  & \ding{51}
  & \multicolumn{1}{c}{$8.54$} & \multicolumn{1}{c}{$0.44$} \\
  \arrayrulecolor{gray!75}\cmidrule(lr){2-5}
  \arrayrulecolor{black}
  & $\|\bm{W}\mX - \bm{W}'\mX'\|_F^2$ & \ding{55}
  & \multicolumn{1}{l}{$\mathbf{12.19}$} & \multicolumn{1}{l}{$\mathbf{0.38}$} \\
  & & \ding{51}
  & \multicolumn{1}{c}{$8.52$} & \multicolumn{1}{c}{$0.43$} \\
\bottomrule
\end{tabular}%
}
\end{wraptable}
\paragraph{Impact of linear layer compression objective and refinement.}

Table~\ref{tab:ablation_objective} isolates the contributions of the layer-wise objective and the block-level refinement step across four objective variants and two compression ratios on LLaMA-7B.
Without refinement, the input-agnostic objective ($\|\bm{W} - \bm{W}'\|_F^2$) is completely degenerate at both ratios, confirming that minimizing weight distance in isolation is insufficient for preserving model behavior.
The input-aware ($\|\bm{W}\mX - \bm{W}'\mX\|_F^2$) and the adaptive and anchored ($\|\bm{W}\mX - \bm{W}'\mX'\|_F^2$) objectives both recover reasonable performance without refinement, with the adaptive and anchored objective performing slightly better in that setting.
The shift-aware objective ($\|\bm{W}\mX' - \bm{W}'\mX'\|_F^2$) performs comparably to the input-aware variant without refinement but does not surpass it after refinement, suggesting that anchoring the reference at the compressed input does not provide additional benefit once block-level optimization is applied.
Refinement consistently and substantially improves all objectives: most strikingly, it rescues the input-agnostic objective from degeneracy and lifts the input-aware objective to the best overall performance at both ratios---$6.89$ PPL and $0.50$ accuracy at ratio $0.8$, and $8.35$ PPL at ratio $0.6$---establishing that block-level joint optimization is the dominant factor driving final quality.
These results indicate that the choice of layer-wise objective matters both with and without refinement: without refinement, it directly determines compression quality, while with refinement it serves as the initialization for block-level refinement, and the final performance remains sensitive to this initialization. Accordingly, our method pairs the input-aware objective with block-level refinement.
\paragraph{Impact of Number of Calibration Samples.}
Figure~\ref{fig:calib_samples} shows the effect of calibration set size on WikiText2 perplexity (left), C4 perplexity (middle), and average zero-shot accuracy (right) at ratios $0.8$ and $0.6$.
Perplexity and accuracy exhibit qualitatively different behaviors: perplexity drops sharply with the first $\sim$64 samples and largely saturates thereafter at both ratios, indicating that a small calibration set is sufficient for good language modeling fidelity.
Accuracy, by contrast, continues to improve substantially beyond 64 samples---particularly at ratio $0.8$, where accuracy rises steeply between 64 and 128 samples before gradually plateauing---suggesting that downstream task performance is more sensitive to calibration budget.
Our default of 256 samples strikes a practical balance: perplexity has saturated and accuracy is near its plateau, while the calibration cost remains modest.

\begin{figure}[t]
    \centering
    \includegraphics[width=1.\linewidth]{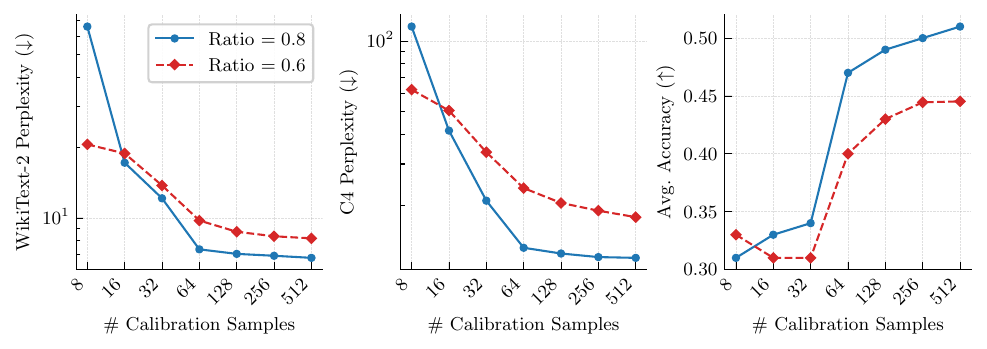}
    \caption{Impact of calibration set size on compression performance. Performance is measured by perplexity on WikiText2 (left) and C4 (middle), and average accuracy across seven zero-shot reasoning tasks (right).}
    \label{fig:calib_samples}
\end{figure}
%
\paragraph{Error Evolution Across Layers.}
To better understand how compression affects internal representations, we track the MSE and cosine 
distance between original and compressed outputs across depth, compressing LLaMA-7B at ratio $0.8$ 
and evaluating on WikiText2 test split samples not used for calibration.
Figure~\ref{fig:error_evolution} reports results for attention output projections, MLP-down projections, and full block outputs, comparing \ourmethodlight~against naive SVD and SVD-LLM (see also Figure~\ref{fig:main_page_figure}).
Naive SVD fails catastrophically from the very first layers: cosine distance immediately saturates 
near $1.0$, indicating that compressed outputs are nearly orthogonal to the originals, with MSE 
several orders of magnitude above both data-driven methods.
SVD-LLM recovers substantially, but its MSE and cosine distance still grow monotonically with depth 
and exhibit notable spikes in the deeper layers.
\ourmethodlight~achieves the lowest error throughout the network on both metrics: cosine distance 
rises gradually to $\sim$0.1--0.15 and then plateaus in later layers, 
while MSE remains consistently below SVD-LLM across all layer indices.
The block-output MSE tells a similar story at a larger scale, reflecting cumulative error across 
all compressed layers within each transformer block---\ourmethodlight~grows more slowly and maintains 
a clear margin over SVD-LLM at every depth.
These results confirm that anchoring compression to the original outputs while accounting for 
shifted inputs curbs error accumulation across depth, which directly underlies the stronger end-task performance observed in the main results.

\begin{figure*}[t]
    \centering
    \begin{subfigure}{0.33\linewidth}
        \centering
        \includegraphics[width=\linewidth]{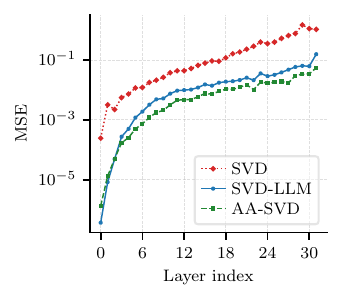}
        \caption{O proj (MSE)}
        \label{fig:q_rel_norm}
    \end{subfigure}%
    \hspace{0.003\linewidth}%
    \begin{subfigure}{0.33\linewidth}
        \centering
        \includegraphics[width=\linewidth]{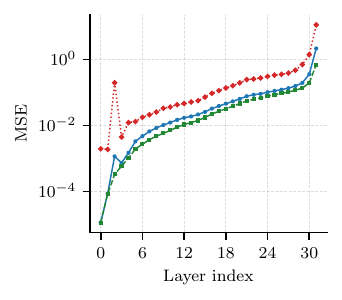}
        \caption{MLP down (MSE)}
        \label{fig:mlp_rel_norm}
    \end{subfigure}%
    \hspace{0.003\linewidth}%
    \begin{subfigure}{0.33\linewidth}
        \centering
        \includegraphics[width=\linewidth]{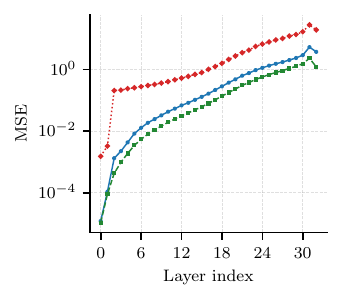}
        \caption{Block outputs (MSE)}
        \label{fig:block_rel_norm}
    \end{subfigure}

    \vspace{1mm}
    \begin{subfigure}{0.33\linewidth}
        \centering
        \includegraphics[width=\linewidth]{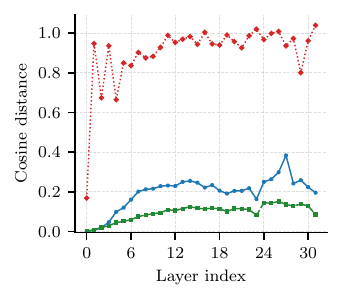}
        \caption{O proj (cosine distance)}
        \label{fig:q_cosine}
    \end{subfigure}%
    \hspace{0.003\linewidth}%
    \begin{subfigure}{0.33\linewidth}
        \centering
        \includegraphics[width=\linewidth]{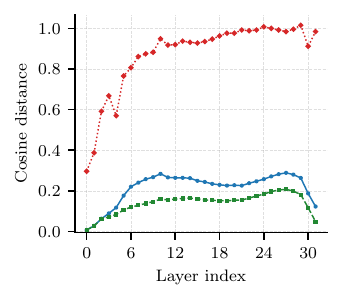}
        \caption{MLP down (cosine distance)}
        \label{fig:mlp_cosine}
    \end{subfigure}%
    \hspace{0.003\linewidth}%
    \begin{subfigure}{0.33\linewidth}
        \centering
        \includegraphics[width=\linewidth]{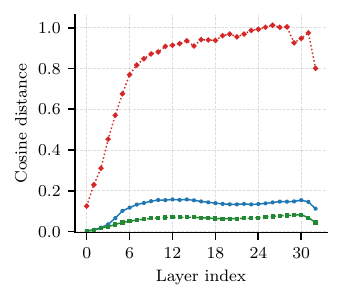}
        \caption{Block outputs (cosine distance)}
        \label{fig:block_cosine}
    \end{subfigure}
    \caption{Layer-wise error evolution across LLaMA-7B at ratio $0.8$, evaluated on WikiText2 test split samples. Top row: MSE between original and compressed outputs. Bottom row: cosine distance between original and compressed outputs. Results are shown separately for attention output projections (O-proj), MLP-down projections, and full block outputs.}
    \label{fig:error_evolution}
\end{figure*}

\section{Limitations and Future Work}
\label{sec:limitations}

Our current method applies a uniform compression ratio across all layers, with per-layer ranks derived from that ratio rather than optimized individually.
Existing capacity allocation methods are not directly applicable to our setting: sensitivity-based approaches evaluate each layer in isolation, but our final compressed factors emerge from block-level joint refinement, making per-layer sensitivity a poor proxy for the true cost of reducing a layer's capacity.
More broadly, because our sequential compression propagates compressed activations forward, the optimal ratio for any layer depends on how upstream layers were compressed — a dependency that per-layer allocation schemes do not account for.
Developing capacity allocation strategies that are aware of block-level refinement is a promising direction for further improving compression performance at a fixed parameter budget.

Beyond capacity allocation, our method also focuses exclusively on low-rank factorization.
Structured pruning removes entire heads or channels, quantization reduces numerical precision, and low-rank factorization compresses weight matrices — each targeting a distinct source of redundancy.
These techniques are largely orthogonal and can in principle be composed: for instance, low-rank factors could be quantized post-compression, or pruning could be applied to reduce capacity before factorization.
Exploring hybrid pipelines that combine these methods, potentially within a unified block-level optimization framework, is a natural and promising direction for future work.

\section{Conclusion}
\label{sec:conclusion}

We introduced a post-training framework for compressing large language models based on low-rank matrix factorization, with two key contributions.
First, we proposed an anchored and adaptive layer-wise compression objective with a closed-form solution, efficient to compute with a small calibration set.
Second, and more critically, we introduced block-level refinement, which jointly optimizes all compressed layers within a transformer block to minimize the block output error---accounting for interactions between compressed layers that layer-wise objectives cannot capture.
Extensive experiments across LLaMA and Qwen model families at multiple scales, evaluated on language modeling and commonsense reasoning benchmarks, show that our approach consistently outperforms prior SVD-based methods.
At moderate compression ratios our method preserves accuracy with negligible loss, while under aggressive compression it widens the gap to baselines.
Overall, our study demonstrates that block-level refinement is the dominant factor enabling effective compression, and that careful initialization via a good layer-wise objective further improves final performance.
Together, these results establish block-level optimization as a powerful and underexplored paradigm for structured model compression. We hope it contributes toward practical deployment of large-scale pretrained models and inspires further work in this direction.




\bibliography{main}
\bibliographystyle{tmlr}

\newpage
\appendix
\section{Proofs and Discussion}
\label{sec:proof-lowrank}

\subsection{Proof of Theorem~\ref{thm:lowrank-cholesky-unreg}}

\paragraph{Problem statement.}
Let $\mW\in\mathbb{R}^{m\times n}$ be a fixed weight matrix and $\mA,\mB\in\mathbb{R}^{n\times l}$ be any two matrices with $\mB\mB^\top$ invertible.
Fix a target rank $k\in\mathbb{N}$.
We seek a solution to
\[
\min_{\operatorname{rank}(\mW')\le k}\;\bigl\|\,\mW\mA - \mW'\mB\,\bigr\|_F^2.
\]
We claim that, for any invertible $\mL_\mB$ satisfying $\mB\mB^\top = \mL_\mB\mL_\mB^\top$, an optimal solution is
\[
\mW'^{\star} \;=\; \operatorname{SVD}_k\!\Bigl(\mW\mA\mB^{\top}\bigl(\mB\mB^{\top}\bigr)^{-1}\mL_\mB\Bigr)\,\mL_\mB^{-1}.
\]

\paragraph{Proof.}
Since $\mB\mB^\top$ is symmetric positive definite, an invertible $\mL_\mB$ with $\mB\mB^\top = \mL_\mB\mL_\mB^\top$ always exists; two concrete choices are (i) the lower-triangular Cholesky factor, or (ii) $\mQ\mLambda^{1/2}$ from the eigendecomposition $\mB\mB^\top = \mQ\mLambda\mQ^\top$.

Expanding the squared Frobenius norm gives
\[
\|\mW\mA - \mW'\mB\|_F^2
= \operatorname{tr}(\mW'\mB\mB^{\top}\mW'^{\top}) - 2\,\operatorname{tr}(\mW\mA\mB^{\top}\mW'^{\top}) + \|\mW\mA\|_F^2.
\]
Since $\mB\mB^\top = \mL_\mB\mL_\mB^\top$, the first term equals $\|\mW'\mL_\mB\|_F^2$. Completing the square yields
\[
\|\mW'\mL_\mB - \mW\mA\mB^{\top}\mL_\mB^{-\top}\|_F^2 - \|\mW\mA\mB^{\top}\mL_\mB^{-\top}\|_F^2 + \|\mW\mA\|_F^2.
\]
Setting $\mM := \mW\mA\mB^\top(\mB\mB^\top)^{-1}\mL_\mB = \mW\mA\mB^\top\mL_\mB^{-\top}$, minimizing \eqref{eq:obj-unreg} is equivalent to minimizing $\|\mW'\mL_\mB - \mM\|_F^2$ subject to $\operatorname{rank}(\mW'\mL_\mB)\le k$.
Because $\mL_\mB$ is invertible, $\operatorname{rank}(\mW'\mL_\mB)=\operatorname{rank}(\mW')$.
By Lemma~\ref{lem:eym}, the unique minimizer of $\|\mZ - \mM\|_F^2$ over rank-$k$ matrices $\mZ$ is $\operatorname{SVD}_k(\mM)$, so $\mW'\mL_\mB = \operatorname{SVD}_k(\mM)$, giving
\[
\mW'^{\star} = \operatorname{SVD}_k(\mM)\,\mL_\mB^{-1},
\]
and the minimal value $\|\mW\mA\|_F^2 - \|\mM\|_F^2 + \sum_{i>k}\sigma_i(\mM)^2$ follows directly from Lemma~\ref{lem:eym}. \qed
\\

\begin{remark*}[Rank-deficient $\mB$]
If $\mB\mB^\top$ is singular, an invertible $\mL_\mB$ satisfying $\mB\mB^\top = \mL_\mB\mL_\mB^\top$ does not exist.
In this case replace $\mL_\mB^{-1}$ by the Moore--Penrose factor $(\mB\mB^\top)^{+1/2}$,
or equivalently use a Tikhonov-regularized factorization
$\mB\mB^\top+\varepsilon \mI = \mL_\varepsilon\mL_\varepsilon^\top$ and let
$\varepsilon\to 0^+$. The same argument then shows that
$\mW'^{\star} = \operatorname{SVD}_k(\mM)\,(\mB\mB^\top)^{+1/2}$
with $\mM := \mW\mA\mB^\top(\mB\mB^\top)^{+1/2}$
is a minimum-norm optimizer, with minimal value given by the same formula.
\end{remark*}

\subsection{Discussion of Corollary~\ref{cor:no-shift}}

Corollary~\ref{cor:no-shift} applies whenever $\mA = \mB$, so that $\mA\mB^\top = \mB\mB^\top = \mL_\mB\mL_\mB^\top$ and the general solution reduces to the whitening-based form
\[
\mW'^{\star} = \operatorname{SVD}_k\!\bigl(\mW\mL_\mB\bigr)\,\mL_\mB^{-1}.
\]
Two natural instantiations arise in our setting: setting $\mB = \mX$ (original inputs) yields an input-aware solution, while setting $\mB = \mX'$ (shifted inputs) gives a shift-aware variant that adapts to the upstream-compressed distribution.
SVD-LLM~\citep{wang2025svdllm} and SVD-LLM V2~\citep{wang2025svdllmv2} both correspond to the $\mB = \mX$ case, differing only in their factorization of $\mX\mX^\top$: SVD-LLM uses the lower-triangular Cholesky factor $\mL_\mX$, while SVD-LLM V2 uses the eigendecomposition $\mX\mX^\top = \mQ\mLambda\mQ^\top$ with $\mL_\mX = \mQ\mLambda^{1/2}$, giving
\[
\mW'^{\star} = \operatorname{SVD}_k\!\bigl(\mW\mQ\mLambda^{1/2}\bigr)\,\mLambda^{-1/2}\mQ^\top.
\]

Since the official SVD-LLM V2 implementation was not publicly available at the time of 
writing, we reproduced it from the paper description. Our reproduction showed no discernible 
performance difference relative to SVD-LLM under either homogeneous or heterogeneous compression 
ratio settings; we therefore report SVD-LLM as the representative baseline for this line of work. 
This choice is consistent with more recent methods, including DipSVD~\citep{ding2025dipsvd} 
and SAES-SVD~\citep{hu2026saes}, which similarly do not report V2 results.

\section{Implementation Details}
\label{sec:appendix_implementation}

\subsection{Linear Layer Compression}
\label{sec:appendix_linear}

Theorem~\ref{thm:lowrank-cholesky-unreg} establishes that the optimal
rank-$k$ compressed operator is obtained by whitening the modified
inputs $\mX'$ via their covariance, projecting the cross-term $\mW\mX$ into
this whitened space, applying truncated SVD, and mapping back. This
closed-form solution generalizes the classical whitening construction
($\mX'=\mX$) and can be implemented efficiently with a Cholesky
factorization. Importantly, our formulation operates only on the
covariance matrices $\mX\mX'^{\top}$ and $\mX'\mX'^{\top}$ rather than the raw
activations themselves. This is especially advantageous when the number
of samples is large (e.g.\ in our setting with $256$ samples of length
$2048$, corresponding to over half a million effective columns), since
the covariance matrices are fixed-size $d\times d$ regardless of the
batch length.

The pseudocode in Algorithm~\ref{alg:chol-compression} details the 
implementation of the linear layer compression step for a single layer, 
which is applied sequentially across layers within each block. In practice,
covariance matrices being computed in Step 2 can be implemented efficiently
in batches and by additively accumulating the outer products ($\mX\mX^{\top}$, $\mX\mX'^{\top}$ and $\mX'\mX'^{\top}$), 
without explicitly materializing the full activation matrices. 
The Cholesky or eigenvalue decomposition in Step 3 is efficient for the moderate hidden 
dimensions of interest (e.g.\ $d=2^{12}-2^{16}~$) with modern GPU-accelerated linear algebra libraries. 
Further, multiple linear layers can share the same covariance matrix if they operate on the same input distribution (e.g.\ query,  key and value projections or MLP gate and up projections within the same block), so the covariance can be reused across layers to amortize the cost of Step 2 and Step 3.

\subsection{Block-level Local Refinement}
\label{sec:appendix_block}
The block-level refinement step (Step~9 of Algorithm~\ref{alg:aa-svd-full}) jointly optimizes the low-rank factors $\{\mU_j, \mV_j\}$ and block-local parameters $\boldsymbol{\theta}_i$ to minimize the MSE between the original and compressed block outputs, as described in the main text.
The objective is minimized via gradient-based optimization: we use the AdamW optimizer~\citep{loshchilov2017decoupled} with a learning rate of $10^{-4}$, trained for 25 epochs over the calibration data with a cosine learning rate schedule and linear warmup, with a batch size of 32.
In our experiments, we find this training configuration to be effective across model families and compression ratios, providing a good balance between recovery quality and computational cost.

Several steps of Algorithm~\ref{alg:aa-svd-full} also admit straightforward implementation optimizations.
Steps~1 and~10 compute the block-input activations $\mX$ and $\mX'$ for the original and compressed models, respectively; the size of these tensors scales with the number of calibration sequences and their length (e.g.\ $N_{\mathrm{cal}} \times L \times d$), and can exceed GPU memory for larger calibration sets.
In practice, these forward passes can be executed in batches on GPU with the resulting activations offloaded to CPU memory between blocks, keeping peak GPU memory usage bounded.
Finally, since the block-level refinement in Step~9 is optimized via standard backpropagation, it can be carried out over batches of calibration sequences on GPU.

\subsection{Memory and Speedup}
\label{sec:appendix_memory_speedup}

Low-rank factorization reduces both parameter count and compute cost by replacing a dense matrix with the product of two thin factors.
Consider a linear layer $\mW \in \mathbb{R}^{m \times n}$.
The original layer requires $mn$ parameters and $O(mn)$ FLOPs per forward pass.
A rank-$k$ factorization stores $mk + nk$ parameters and incurs $O(mk + nk)$ FLOPs, which is cheaper whenever $k \ll \min(m,n)$.
The effective compression ratio is
\[
\rho = \frac{mk + nk}{mn}.
\]
For example, with $m=n=4096$ and $k=512$ ($\rho=0.125$), the parameter count drops from $16.8$M to $4.2$M (a $4\times$ reduction), and FLOPs per forward pass reduce by the same factor.

Beyond weights and FLOPs, low-rank factorization can also reduce the memory footprint of the key–value (KV) cache during autoregressive inference.
Since attention projections are compressed, the activations stored in the cache scale with $k$ rather than $n$, yielding proportional savings in both memory and bandwidth.
As highlighted in SVD-LLM~\citep{wang2025svdllm} and follow-up works \citep{Wang2025DobiSVD,hu2026saes}, this reduction is crucial for long-context inference where KV-cache dominates memory usage.

Our method (\ourmethodlight) preserves this structural efficiency: the cost of computing compressed weights is incurred once during compression, while inference cost and KV-cache size match those of standard low-rank layers.
Thus, \ourmethodlight~offers the same runtime and memory benefits as prior SVD-based methods, with its main advantage lying in improved approximation quality under aggressive compression.





\subsection{Dobi-SVD Remapping}
\label{sec:appendix_dobi_svd_remapping}

Standard SVD-based compression stores a rank-$k$ approximation of an $m \times n$ weight matrix as two dense factors of total size $k(m+n)$, giving a compression ratio $\rho = k(m+n)/(mn)$.
Dobi-SVD~\citep{Wang2025DobiSVD} proposes a \emph{remapping} that stores the smaller factor and the top $\min(m,n)$ rows/columns of the larger factor in half precision (16-bit $\to$ 8-bit), with the remaining $(\max(m,n) - \min(m,n))$ rows/columns kept in full precision.
The total storage in full-precision-equivalent units reduces to $\max(m,n) \cdot k (=0.5 \cdot 2\min(m,n) \cdot k + (\max(m,n) - \min(m,n)) \cdot k)$.
This gives a compression ratio of $\rho = \max(m,n) \cdot k / (mn) = k / \min(m,n)$, so that every target ratio $\rho \in [0, 1]$ maps to a unique truncation rank $k = \rho \cdot \min(m,n)$, spanning the full valid range $k \in [0, \min(m,n)]$\footnote{Under the standard formula, $\rho \leq 1$ restricts $k \leq mn/(m+n)$, precluding high-rank approximations.}.

Because this remapping changes what a stated compression ratio means in terms of actual parameter counts, a direct comparison between Dobi-SVD and methods using the standard formula at the same nominal ratio is unfair.
To address this, we report results both without remapping (standard formula, comparable across all methods) and with remapping enabled for \ourmethodlight, denoted \ourmethodlight$^q$, at the same effective parameter budget as Dobi-SVD$^{\ddagger}$ and Dobi-SVD$^{\ddagger, q}$.

\section{Compression performance on more models}
\label{sec:comparison_more_models}

Tables~\ref{tab:LLaMA-3-1B}--\ref{tab:LLaMA-2-13B} provide full per-benchmark breakdowns for the five additional models summarized in Table~\ref{tab:multiple_models_summary} of the main text.
The results consistently replicate the trends observed on LLaMA-7B, confirming that the gains from block-level joint optimization generalize across model families (LLaMA-2, LLaMA-3, Qwen-2.5) and scales (1B--13B parameters).
SVD-LLM results are reproduced by us.
For other baselines, numbers are taken from their respective papers where available for the given model and compression ratio; entries are left blank where results were not reported.

\begin{table*}[t]
\centering
\caption{Comparison of \ourmethod~with SOTA methods for SVD-based compression of LLaMA-3-1B on three language modeling tasks and seven commonsense reasoning benchmarks (zero-shot evaluation). Best performance is marked in bold. \\
}
\label{tab:LLaMA-3-1B}
\resizebox{\textwidth}{!}{
\begin{tabular}{c l ccc !{\hspace{3pt}\vline\hspace{3pt}} ccccccccc}
\toprule
\multirow{2}{*}{\textbf{Ratio}} & \multirow{2}{*}{\textbf{Method}} & \multicolumn{3}{c}{\textbf{PPL} ($\downarrow$)} & \multicolumn{9}{c}{\textbf{Accuracy} ($\uparrow$)} \\
\cmidrule(lr){3-5} \cmidrule(lr){6-14}
 & & \textbf{Wiki2} & \textbf{PTB} & \textbf{C4} & \textbf{Openb.} & \textbf{ARC\_e} & \textbf{ARC\_c} & \textbf{WinoG.} & \textbf{PIQA} & \textbf{MathQA} & \textbf{HellaS.} & \textbf{Avg.} & \textbf{Drop (\%)} \\
\specialrule{\lightrulewidth}{2pt}{2pt}
$1.0$ & Dense & $9.75$ & $15.40$ & $13.82$ & $0.26$ & $0.65$ & $0.31$ & $0.61$ & $0.74$ & $0.29$ & $0.48$ & $0.48$ & $-$ \\
\specialrule{\lightrulewidth}{2pt}{2pt}
\multirow{1}{*}{$0.8$}
& SVD-LLM & $45.62$ & $158.15$ & $206.18$ & $0.14$ & $0.37$ & $0.19$ & $0.51$ & $0.56$ & $0.21$ & $0.28$ & $0.32$ & $32.3\%$ \\
\rowcolor{gray!12} \cellcolor{white} & \ourmethodlight~& $\mathbf{15.12}$ & $\mathbf{36.81}$ & $\mathbf{37.54}$ & $\mathbf{0.20}$ & $\mathbf{0.51}$ & $\mathbf{0.23}$ & $\mathbf{0.55}$ & $\mathbf{0.64}$ & $\mathbf{0.23}$ & $\mathbf{0.36}$ & $\mathbf{0.39}$ & $\mathbf{18.6\%}$ \\
\specialrule{\lightrulewidth}{2pt}{2pt}
\multirow{1}{*}{$0.6$}
& SVD-LLM & $402.76$ & $2027.07$ & $1449.82$ & $0.12$ & $0.27$ & $0.20$ & $0.51$ & $0.52$ & $0.20$ & $0.26$ & $0.30$ & $37.7\%$ \\
\rowcolor{gray!12} \cellcolor{white} & \ourmethodlight~& $\mathbf{23.74}$ & $\mathbf{72.00}$ & $\mathbf{91.02}$ & $\mathbf{0.19}$ & $\mathbf{0.42}$ & $\mathbf{0.22}$ & $\mathbf{0.53}$ & $\mathbf{0.58}$ & $\mathbf{0.23}$ & $\mathbf{0.30}$ & $\mathbf{0.35}$ & $\mathbf{26.1\%}$ \\
\specialrule{\lightrulewidth}{2pt}{2pt}
\multirow{1}{*}{$0.4$}
& SVD-LLM & $1369.77$ & $5082.80$ & $3520.70$ & $0.13$ & $0.26$ & $\mathbf{0.21}$ & $0.51$ & $0.53$ & $0.20$ & $0.26$ & $0.30$ & $37.1\%$ \\
\rowcolor{gray!12} \cellcolor{white} & \ourmethodlight~& $\mathbf{51.01}$ & $\mathbf{192.65}$ & $\mathbf{246.63}$ & $\mathbf{0.16}$ & $\mathbf{0.35}$ & ${0.20}$ & $\mathbf{0.52}$ & $\mathbf{0.56}$ & $\mathbf{0.21}$ & $\mathbf{0.27}$ & $\mathbf{0.32}$ & $\mathbf{32.1\%}$ \\
\bottomrule
\end{tabular}
}
\end{table*}

\begin{table*}[t]
\centering
\caption{Comparison of \ourmethod~with SOTA methods for SVD-based compression of LLaMA-2-7B on three language modeling tasks and seven commonsense reasoning benchmarks (zero-shot evaluation). Best performance is marked in bold. $^\ddagger$ uses dynamic or non-uniform ratio allocation, and $^q$ represents quantized parameters.\\\\
}
\label{tab:LLaMA-2-7B}
\resizebox{\textwidth}{!}{
\begin{tabular}{c l ccc !{\hspace{3pt}\vline\hspace{3pt}} ccccccccc}
\toprule
\multirow{2}{*}{\textbf{Ratio}} & \multirow{2}{*}{\textbf{Method}} & \multicolumn{3}{c}{\textbf{PPL} ($\downarrow$)} & \multicolumn{9}{c}{\textbf{Accuracy} ($\uparrow$)} \\
\cmidrule(lr){3-5} \cmidrule(lr){6-14}
 & & \textbf{Wiki2} & \textbf{PTB} & \textbf{C4} & \textbf{Openb.} & \textbf{ARC\_e} & \textbf{ARC\_c} & \textbf{WinoG.} & \textbf{PIQA} & \textbf{MathQA} & \textbf{HellaS.} & \textbf{Avg.} & \textbf{Drop (\%)} \\
\specialrule{\lightrulewidth}{2pt}{2pt}
$1.0$ & Dense & $5.47$ & $24.09$ & $7.28$ & $0.32$ & $0.76$ & $0.43$ & $0.69$ & $0.78$ & $0.28$ & $0.57$ & $0.55$ & $-$ \\
\specialrule{\lightrulewidth}{2pt}{2pt}
\multirow{1}{*}{$0.8$}
& SVD-LLM & $8.41$ & $119.32$ & $20.34$ & $0.26$ & $0.57$ & $0.26$ & $0.62$ & $0.66$ & $0.24$ & $0.39$ & $0.43$ & $21.7\%$ \\
\rowcolor{gray!12} \cellcolor{white} & \ourmethodlight~& $\mathbf{6.84}$ & $\mathbf{1486.20}$ & $\mathbf{13.19}$ & $\mathbf{0.30}$ & $\mathbf{0.71}$ & $\mathbf{0.37}$ & $\mathbf{0.64}$ & $\mathbf{0.72}$ & $\mathbf{0.27}$ & $\mathbf{0.48}$ & $\mathbf{0.50}$ & $\mathbf{8.9\%}$ \\
\arrayrulecolor{gray!75}\cmidrule(lr){2-14}
\arrayrulecolor{black}
& {Dobi-SVD$^{\ddagger, q}$} & {$5.92$} & $-$ & $-$ & $-$ & $-$ & $-$ & $-$ & $-$ & {$-$} & $-$ & $-$ & $-$ \\
\rowcolor{gray!12} \cellcolor{white} & \ourmethodlight$^q$ & $5.92$ & $30.78$ & $8.41$ & $0.31$ & $0.74$ & $0.42$ & $0.69$ & $0.77$ & $0.29$ & $0.55$ & $0.54$ & $\mathbf{1.6\%}$ \\
\specialrule{\lightrulewidth}{2pt}{2pt}
\multirow{2}{*}{$0.6$}
& SVD-LLM & $16.47$ & $571.51$ & $73.12$ & $0.18$ & $0.39$ & $0.21$ & $0.53$ & $0.58$ & $0.22$ & $0.31$ & $0.35$ & $36.8\%$ \\
& SAES-SVD & $11.35$ & $\mathbf{217.20}$ & $40.57$ & $-$ & $0.43$ & $-$ & $0.58$ & ${0.59}$ & $-$ & $0.32$ & $-$ & $-$ \\
\rowcolor{gray!12} \cellcolor{white} & \ourmethodlight~& $\mathbf{8.55}$ & ${2688.10}$ & $\mathbf{21.78}$ & $\mathbf{0.27}$ & $\mathbf{0.60}$ & $\mathbf{0.30}$ & $\mathbf{0.62}$ & $\mathbf{0.66}$ & $\mathbf{0.25}$ & $\mathbf{0.41}$ & $\mathbf{0.44}$ & $\mathbf{18.8\%}$ \\
\arrayrulecolor{gray!75}\cmidrule(lr){2-14}
\arrayrulecolor{black}
& {Dobi-SVD$^{\ddagger, q}$} & {$7.88$} & $-$ & $-$ & $-$ & $0.67$ & $0.31$ & $0.64$ & $0.72$ & {$-$} & $0.45$ & $-$ & $-$ \\
\rowcolor{gray!12} \cellcolor{white} & \ourmethodlight$^q$ & $6.77$ & $100.25$ & $11.64$ & $0.29$ & $0.72$ & $0.39$ & $0.66$ & $0.73$ & $0.28$ & $0.50$ & $0.51$ & $\mathbf{6.8\%}$ \\
\specialrule{\lightrulewidth}{2pt}{2pt}
\multirow{2}{*}{$0.4$}
& SVD-LLM & $97.43$ & $1612.91$ & $615.24$ & $0.13$ & $0.27$ & $0.22$ & $0.49$ & $0.53$ & $0.20$ & $0.27$ & $0.30$ & $44.9\%$ \\
& SAES-SVD & $23.89$ & $\mathbf{334.67}$ & $100.42$ & $-$ & $0.31$ & $-$ & $0.52$ & $0.55$ & $-$ & $0.30$ & $-$ & $-$ \\
\rowcolor{gray!12} \cellcolor{white} & \ourmethodlight~& $\mathbf{14.58}$ & $4342.20$ & $\mathbf{53.22}$ & $\mathbf{0.20}$ & $\mathbf{0.44}$ & $\mathbf{0.24}$ & $\mathbf{0.56}$ & $\mathbf{0.60}$ & $\mathbf{0.24}$ & $\mathbf{0.32}$ & $\mathbf{0.37}$ & $\mathbf{32.1\%}$ \\
\arrayrulecolor{gray!75}\cmidrule(lr){2-14}
\arrayrulecolor{black}
& {Dobi-SVD$^{\ddagger, q}$} & {$9.47$} & $-$ & $-$ & $-$ & $0.55$ & $0.26$ & $0.57$ & $0.67$ & {$-$} & $0.38$ & $-$ & $-$ \\
\rowcolor{gray!12} \cellcolor{white} & \ourmethodlight$^q$ & $8.86$ & $528.41$ & $22.48$ & $0.26$ & $0.60$ & $0.30$ & $0.61$ & $0.65$ & $0.25$ & $0.40$ & $0.44$ & $\mathbf{19.8\%}$ \\
\bottomrule
\end{tabular}
}
\end{table*}

\begin{table*}[t]
\centering
\caption{Comparison of \ourmethod~with SOTA methods for SVD-based compression of LLaMA-3-8B on three language modeling tasks and seven commonsense reasoning benchmarks (zero-shot evaluation). Best performance is marked in bold.\\
}
\label{tab:LLaMA-3-8B}
\resizebox{\textwidth}{!}{
\begin{tabular}{c l ccc !{\hspace{3pt}\vline\hspace{3pt}} ccccccccc}
\toprule
\multirow{2}{*}{\textbf{Ratio}} & \multirow{2}{*}{\textbf{Method}} & \multicolumn{3}{c}{\textbf{PPL} ($\downarrow$)} & \multicolumn{9}{c}{\textbf{Accuracy} ($\uparrow$)} \\
\cmidrule(lr){3-5} \cmidrule(lr){6-14}
 & & \textbf{Wiki2} & \textbf{PTB} & \textbf{C4} & \textbf{Openb.} & \textbf{ARC\_e} & \textbf{ARC\_c} & \textbf{WinoG.} & \textbf{PIQA} & \textbf{MathQA} & \textbf{HellaS.} & \textbf{Avg.} & \textbf{Drop (\%)} \\
\specialrule{\lightrulewidth}{2pt}{2pt}
$1.0$ & Dense & $6.24$ & $9.89$ & $9.57$ & $0.34$ & $0.82$ & $0.52$ & $0.74$ & $0.80$ & $0.40$ & $0.60$ & $0.60$ & $-$ \\
\specialrule{\lightrulewidth}{2pt}{2pt}
\multirow{2}{*}{$0.8$}
& SVD-LLM & $14.16$ & $64.01$ & $79.14$ & $0.24$ & $0.59$ & $0.30$ & $0.64$ & $0.66$ & $0.26$ & $0.37$ & $0.44$ & $27.5\%$ \\
& SAES-SVD & $11.49$ & $-$ & $-$ & $0.25$ & $0.59$ & ${0.28}$ & $0.66$ & ${0.67}$ & $0.27$ & $0.39$ & $0.44$ & $26.3\%$ \\
\rowcolor{gray!12} \cellcolor{white} & \ourmethodlight~& $\mathbf{9.58}$ & $\mathbf{28.11}$ & $\mathbf{33.12}$ & $\mathbf{0.26}$ & $\mathbf{0.70}$ & $\mathbf{0.37}$ & $\mathbf{0.69}$ & $\mathbf{0.72}$ & $\mathbf{0.30}$ & $\mathbf{0.46}$ & $\mathbf{0.50}$ & $\mathbf{17.1\%}$ \\
\specialrule{\lightrulewidth}{2pt}{2pt}
\multirow{2}{*}{$0.6$}
& SVD-LLM & $76.31$ & $971.56$ & $662.65$ & $0.14$ & $0.32$ & $0.19$ & $0.52$ & $0.55$ & $0.21$ & $0.28$ & $0.32$ & $47.6\%$ \\
& SAES-SVD & $23.30$ & $-$ & $-$ & $0.16$ & $0.34$ & $0.20$ & $0.55$ & ${0.55}$ & $0.22$ & $0.30$ & $0.33$ & $45.0\%$ \\
\rowcolor{gray!12} \cellcolor{white} & \ourmethodlight~& $\mathbf{13.66}$ & $\mathbf{56.33}$ & $\mathbf{74.48}$ & $\mathbf{0.23}$ & $\mathbf{0.54}$ & $\mathbf{0.27}$ & $\mathbf{0.61}$ & $\mathbf{0.64}$ & $\mathbf{0.24}$ & $\mathbf{0.37}$ & $\mathbf{0.41}$ & $\mathbf{31.3\%}$ \\
\specialrule{\lightrulewidth}{2pt}{2pt}
\multirow{2}{*}{$0.4$}
& SVD-LLM & $649.12$ & $8403.95$ & $3375.48$ & $0.12$ & $0.27$ & $0.20$ & $0.51$ & $0.52$ & $0.20$ & $0.26$ & $0.30$ & $50.7\%$ \\
& SAES-SVD & $63.09$ & $-$ & $-$ & $0.13$ & $0.29$ & $\mathbf{0.22}$ & $0.53$ & $0.54$ & $0.23$ & $0.28$ & $0.32$ & $47.4\%$ \\
\rowcolor{gray!12} \cellcolor{white} & \ourmethodlight~& $\mathbf{32.23}$ & $\mathbf{263.02}$ & $\mathbf{323.43}$ & $\mathbf{0.18}$ & $\mathbf{0.38}$ & ${0.20}$ & $\mathbf{0.52}$ & $\mathbf{0.58}$ & $\mathbf{0.22}$ & $\mathbf{0.30}$ & $\mathbf{0.34}$ & $\mathbf{  43.6\%}$ \\
\bottomrule
\end{tabular}
}
\end{table*}

\begin{table*}[t]
\centering
\caption{Comparison of \ourmethod~with SOTA methods for SVD-based compression of Qwen-2.5-7B on three language modeling tasks and seven commonsense reasoning benchmarks (zero-shot evaluation). Best performance is marked in bold.\\
}
\label{tab:qwen-2.5-7B}
\resizebox{\textwidth}{!}{
\begin{tabular}{c l ccc !{\hspace{3pt}\vline\hspace{3pt}} ccccccccc}
\toprule
\multirow{2}{*}{\textbf{Ratio}} & \multirow{2}{*}{\textbf{Method}} & \multicolumn{3}{c}{\textbf{PPL} ($\downarrow$)} & \multicolumn{9}{c}{\textbf{Accuracy} ($\uparrow$)} \\
\cmidrule(lr){3-5} \cmidrule(lr){6-14}
 & & \textbf{Wiki2} & \textbf{PTB} & \textbf{C4} & \textbf{Openb.} & \textbf{ARC\_e} & \textbf{ARC\_c} & \textbf{WinoG.} & \textbf{PIQA} & \textbf{MathQA} & \textbf{HellaS.} & \textbf{Avg.} & \textbf{Drop (\%)} \\
\specialrule{\lightrulewidth}{2pt}{2pt}
$1.0$ & Dense & $6.84$ & $11.37$ & $11.85$ & $0.34$ & $0.80$ & $0.50$ & $0.73$ & $0.79$ & $0.43$ & $0.60$ & $0.60$ & $-$ \\
\specialrule{\lightrulewidth}{2pt}{2pt}
\multirow{1}{*}{$0.8$}
& SVD-LLM & $10.69$ & $39.10$ & $38.53$ & $0.25$ & $0.67$ & $0.31$ & $0.64$ & $0.68$ & $0.32$ & $0.41$ & $0.47$ & $21.7\%$ \\
\rowcolor{gray!12} \cellcolor{white} & \ourmethodlight~& $\mathbf{8.53}$ & $\mathbf{22.90}$ & $\mathbf{22.05}$ & $\mathbf{0.31}$ & $\mathbf{0.74}$ & $\mathbf{0.41}$ & $\mathbf{0.69}$ & $\mathbf{0.73}$ & $\mathbf{0.37}$ & $\mathbf{0.49}$ & $\mathbf{0.53}$ & $\mathbf{10.7\%}$ \\
\specialrule{\lightrulewidth}{2pt}{2pt}
\multirow{1}{*}{$0.6$}
& SVD-LLM & $28.67$ & $193.31$ & $161.22$ & $0.15$ & $0.36$ & $0.20$ & $0.53$ & $0.56$ & $0.22$ & $0.29$ & $0.33$ & $44.9\%$ \\
\rowcolor{gray!12} \cellcolor{white} & \ourmethodlight~& $\mathbf{11.00}$ & $\mathbf{49.10}$ & $\mathbf{40.85}$ & $\mathbf{0.25}$ & $\mathbf{0.59}$ & $\mathbf{0.28}$ & $\mathbf{0.61}$ & $\mathbf{0.65}$ & $\mathbf{0.28}$ & $\mathbf{0.39}$ & $\mathbf{0.44}$ & $\mathbf{27.2\%}$ \\
\specialrule{\lightrulewidth}{2pt}{2pt}
\multirow{1}{*}{$0.4$}
& SVD-LLM & $136.74$ & $963.37$ & $647.59$ & $0.12$ & $0.28$ & $0.20$ & $0.49$ & $0.54$ & $0.21$ & $0.27$ & $0.30$ & $49.6\%$ \\
\rowcolor{gray!12} \cellcolor{white} & \ourmethodlight~& $\mathbf{15.67}$ & $\mathbf{86.23}$ & $\mathbf{62.81}$ & $\mathbf{0.20}$ & $\mathbf{0.44}$ & $\mathbf{0.23}$ & $\mathbf{0.57}$ & $\mathbf{0.60}$ & $\mathbf{0.23}$ & $\mathbf{0.33}$ & $\mathbf{0.37}$ & $\mathbf{37.9\%}$ \\
\bottomrule
\end{tabular}
}
\end{table*}

\begin{table*}[t] 
\centering
\caption{Comparison of \ourmethod~with SOTA methods for SVD-based compression of LLaMA-2-13B on three language modeling tasks and seven commonsense reasoning benchmarks (zero-shot evaluation). \\
}
\label{tab:LLaMA-2-13B}
\resizebox{\textwidth}{!}{
\begin{tabular}{c l ccc !{\hspace{3pt}\vline\hspace{3pt}} ccccccccc}
\toprule
\multirow{2}{*}{\textbf{Ratio}} & \multirow{2}{*}{\textbf{Method}} & \multicolumn{3}{c}{\textbf{PPL} ($\downarrow$)} & \multicolumn{9}{c}{\textbf{Accuracy} ($\uparrow$)} \\
\cmidrule(lr){3-5} \cmidrule(lr){6-14}
 & & \textbf{Wiki2} & \textbf{PTB} & \textbf{C4} & \textbf{Openb.} & \textbf{ARC\_e} & \textbf{ARC\_c} & \textbf{WinoG.} & \textbf{PIQA} & \textbf{MathQA} & \textbf{HellaS.} & \textbf{Avg.} & \textbf{Drop (\%)} \\
\specialrule{\lightrulewidth}{2pt}{2pt}
$1.0$ & Dense & $4.88$ & $34.40$ & $6.73$ & $0.35$ & $0.79$ & $0.48$ & $0.72$ & $0.79$ & $0.32$ & $0.60$ & $0.58$ & $-$ \\
\specialrule{\lightrulewidth}{2pt}{2pt}
\multirow{1}{*}{$0.8$}
& SVD-LLM & $6.65$ & $84.17$ & $14.99$ & $0.29$ & $0.67$ & $0.33$ & $0.68$ & $0.71$ & $0.26$ & $0.44$ & $0.48$ & $16.5\%$ \\
\rowcolor{gray!12} \cellcolor{white} & \ourmethodlight~& $\mathbf{5.95}$ & $\mathbf{46.43}$ & $\mathbf{11.6}$ & $\mathbf{0.33}$ & $\mathbf{0.73}$ & $\mathbf{0.40}$ & $\mathbf{0.69}$ & $\mathbf{0.74}$ & $\mathbf{0.29}$ & $\mathbf{0.52}$ & $\mathbf{0.53}$ & $\mathbf{8.9\%}$ \\
\specialrule{\lightrulewidth}{2pt}{2pt}
\multirow{1}{*}{$0.6$}
& SVD-LLM & $10.79$ & $243.85$ & $46.47$ & $0.22$ & $0.47$ & $0.22$ & $0.61$ & $0.60$ & $0.23$ & $0.33$ & $0.38$ & $33.8\%$ \\
\rowcolor{gray!12} \cellcolor{white} & \ourmethodlight~& $\mathbf{7.44}$ & $\mathbf{79.01}$ & $\mathbf{19.32}$ & $\mathbf{0.25}$ & $\mathbf{0.64}$ & $\mathbf{0.32}$ & $\mathbf{0.63}$ & $\mathbf{0.68}$ & $\mathbf{0.26}$ & $\mathbf{0.41}$ & $\mathbf{0.46}$ & $\mathbf{21.2\%}$ \\
\specialrule{\lightrulewidth}{2pt}{2pt}
\multirow{1}{*}{$0.4$}
& SVD-LLM & $44.28$ & $1296.01$ & $295.21$ & $0.14$ & $0.31$ & $0.20$ & $0.52$ & $0.54$ & $0.21$ & $0.27$ & $0.31$ & $45.9\%$ \\
\rowcolor{gray!12} \cellcolor{white} & \ourmethodlight~& $\mathbf{11.77}$ & $\mathbf{154.96}$ & $\mathbf{42.50}$ & $\mathbf{0.23}$ & $\mathbf{0.45}$ & $\mathbf{0.24}$ & $\mathbf{0.58}$ & $\mathbf{0.60}$ & $\mathbf{0.24}$ & $\mathbf{0.35}$ & $\mathbf{0.38}$ & $\mathbf{33.6\%}$ \\
\bottomrule
\end{tabular}
}
\end{table*}

\end{document}